\documentclass{article}

\usepackage{microtype}
\usepackage{graphicx}
\usepackage{subfigure}
\usepackage{booktabs} 
\usepackage{amssymb,amsmath, amsthm}
\usepackage{epstopdf}
\usepackage{hyperref}
\usepackage{algorithm}
\usepackage{algorithmic}
\usepackage{enumitem} 

\numberwithin{equation}{section}

\usepackage[dvipsnames]{xcolor}

\DeclareMathOperator{\Entropy}{Entropy}
\newcommand{\ip}[1] {\langle #1 \rangle }

\newcommand{\EE}{\mathbb{E}}
\newcommand{\RR}{\mathbb{R}}

\newcommand{\cA}{\mathcal{A}}

\newcommand{\cX}{\mathcal{X}}

\newcommand{\cS}{\mathcal{S}} 
    
\newcommand{\cO}{{\mathcal{O}}}
\newcommand{\cP}{{\mathcal{P}}}
\newcommand{\cY}{\mathcal{Y}}
\newcommand{\cL}{\mathcal{L}} 

\newcommand{\Proj}{\mathrm{Proj}\, }

\newcommand{\argmin}{\operatornamewithlimits{argmin}}
\newcommand{\argmax}{\operatornamewithlimits{argmax}}

\newcommand{\normm}[1]{{\vert\kern-0.25ex\vert\kern-0.25ex\vert #1 \vert\kern-0.25ex\vert\kern-0.25ex\vert}}
\newcommand{\norm}[1]{\|#1\|}

\newcommand{\das}{\Delta_{\cA}^{\cS}}
\newcommand{\F}{F}

\newtheorem{theorem}{Theorem}[section]
\newtheorem{proposition}[theorem]{Proposition}
\newtheorem{lemma}[theorem]{Lemma}

\newtheorem{assumption}[theorem]{Assumption}

\newtheorem{example}{Example}[section]

\usepackage{mathtools}

\usepackage{jmlr2e} 

\begin{document}

\jmlrheading{1}{2020}{pp}{mm/dd}{mm/dd}{Junyu Zhang, Alec Koppel, Amrit Singh Bedi, Csaba Szepesvari, and Mengdi Wang}


\ShortHeadings{Variational Policy Gradient}{Zhang,  Koppel, Bedi, Szepesvari, and Wang}
\firstpageno{1}

\title{Variational Policy Gradient Method for\\ Reinforcement Learning with General Utilities}

\author{\name Junyu Zhang  \email zhan4393@umn.edu\\
\addr Department of Industrial and Systems Engineering\\
 University of Minnesota\\
  Minneapolis, Minnesota, 55455
  \AND
         \name Alec Koppel \email alec.e.koppel.civ@mail.mil\\
         \addr Computational and Information Sciences Directorate\\
          US Army Research Laboratory \\
           Adelphi, MD 20783 
       \AND
       \name Amrit Singh Bedi \email amrit0714@gmail.com\\
       \addr Computational and Information Sciences Directorate\\
        US Army Research Laboratory \\
         Adelphi, MD, USA 20783
         \AND
                 \name Csaba Szepesvari \email szepesva@ualberta.ca  \\
                \addr Department of Computer Science \\
                DeepMind/University of Alberta\\
                Princeton, NJ 08544
              \AND
        \name Mengdi Wang \email mengdiw@princeton.edu  \\
       \addr Department of Electrical Engineering \\
       Center for Statistics and Machine Learning\\
       Princeton University/Deepmind\\
       Princeton, NJ 08544
       }


\editor{}

	\maketitle
	
	\begin{abstract}
	%
In recent years, reinforcement learning (RL) systems with general goals beyond a cumulative sum of rewards have gained traction, such as in constrained problems, exploration, and acting upon prior experiences. In this paper, we consider policy optimization in Markov Decision Problems, where the objective is a general concave utility function of the state-action occupancy measure, which subsumes several of the aforementioned examples as special cases. Such generality invalidates the Bellman equation. As this means that dynamic programming no longer works, we focus on direct policy search. Analogously to the Policy Gradient Theorem \cite{sutton2000policy} available for RL with cumulative rewards, we derive a new Variational Policy Gradient Theorem for RL with general utilities, which establishes that the parametrized policy gradient may be obtained as the solution of a stochastic saddle point problem involving the Fenchel dual of the utility function. We develop a variational Monte Carlo gradient estimation algorithm to compute the policy gradient based on sample paths. We prove that the variational policy gradient scheme converges globally to the optimal policy for the general objective, though the optimization problem is nonconvex. We also establish its rate of convergence of the order $O(1/t)$ by exploiting the hidden convexity of the problem, and proves that it converges exponentially when the problem admits hidden strong convexity. Our analysis applies to the standard RL problem with cumulative rewards as a special case, in which case our result improves the available convergence rate. 
%
	\end{abstract}
\section{Introduction}
The standard formulation of reinforcement learning (RL) is concerned with finding a policy that maximizes the expected sum of rewards along the sample paths generated by the policy. 
The additive nature of the objective function creates an attractive algebraic structure which most efficient RL algorithms exploit. 
However, the cumulative reward objective is not the only one that has attracted attention.
In fact, many alternative objectives made appearances already in the early literature on stochastic optimal control and operations research.
Examples include 
various kinds of risk-sensitive objectives \cite{Ka94I,borkar2002risk,yu2009,mannor2011mean}, 
objectives to maximize the entropy of the state visitation distribution \cite{hazan2018provably},
the incorporation of constraints \cite{DeKl65,altman1999constrained,achiam2017constrained}, 
and learning to ``mimic'' a demonstration  \cite{schaal1997learning,argall2009survey}. 
	
In this paper, we consider RL with general utility functions, and we aim to develop a principled methodology and theory for policy optimization in such problems. We focus on utility functions that are concave functionals of the state-action occupancy measure, which contains many, although not all, of the aforementioned examples as special cases.
The general (or non-standard \cite{Ka94I}) utility is a strict generalization of cumulative reward, which itself can be viewed as a linear functional of the state-action occupancy measure, and as such, is a concave function of the occupancy measures.


	When moving beyond cumulative rewards, we quickly run into technical challenges because of the lack of additive structure. Without additivity of rewards, the problem becomes non-Markovian in the cost-to-go \cite{takacs1966non,whitehead1995reinforcement}. %
Consequently, the Bellman equation fails to hold and dynamic programming (DP) breaks down. Therefore, stochastic methods based upon DP such as temporal difference \cite{sutton1988learning} and Q-learning \cite{watkins1992q,ross2014introduction} are inapplicable.
The value function, the core quantity for RL, is not even well defined for general utilities, 
thus invalidating the foundation of value-function based approach to RL. 
	
	Due to these challenges, we consider direct policy search methods for the solution of RL problems defined by general utility functions. 
	We consider the most elementary policy-based method, namely the Policy Gradient (PG) method \cite{williams1992simple}.
The idea of policy gradient methods is that to represent policies through some policy parameterization and then move the parameters of a policy in the direction of the gradient of the objective function. When (as typical) only a noisy estimate of the gradient is available, we arrive at a stochastic approximation method \cite{robbins1951stochastic,kiefer1952stochastic}.  
In the classical cumulative reward objectives,  the gradient can be written as the product of the action-value function and the gradient of the logarithm of the policy, or policy score function \cite{sutton2000policy}. 
State-of-the-art RL algorithms for the cumulative reward setting combine this result  with other ideas, such as
limiting the changes to the policies \cite{KaLa02,schulman2015trust,schulman2017proximal}, 
variance reduction \cite{kakade2002natural,papini2018stochastic,xu2019sample}, 
or exploiting structural aspects of the policy parameterization \cite{wang2019neural,agarwal2019optimality,mei2020global}.
	
As mentioned, these approaches crucially rely on the standard PG Theorem \cite{sutton2000policy}, which is not available for general utilities. Compounding this challenge is the fact that the action-value function is not well-defined in this instance, either. Thus, how and whether the policy gradient can be effectively computed becomes a question. Further, due to the problem's nonconvexity, it is an open question whether an iterative policy improvement scheme converges to anything meaningful: In particular, while standard results for stochastic approximation would give convergence to stationary points \cite{Bo09}, it is unclear whether the stationary points give reasonable policies. 
Therefore, we ask the question:
		\begin{center}\it
	Is  policy search viable for general utilities, \\when Bellman's equation, the value function, and dynamic programming all fail?
		\end{center}
	We will answer the question positively in this paper. Our contributions are three-folded:	
	\begin{itemize}
	\item We derive a Variational Policy Gradient Theorem for RL with general utilities which establishes that the parametrized policy gradient is the solution to a stochastic saddle point problem. 	
	\item We show that the Variational Policy Gradient can be estimated by a primal-dual stochastic approximation method based on sample paths generated by following the current policy \cite{Arrow1958}. We prove that the random error of the estimate decays at order $O(1/\sqrt{n})$ that also depends on properties of the utility, where $n$ is the number of episodes . 
	\item We consider the non-parameterized policy optimization problem which is nonconvex in the policy space. Despite the lack of convexity, we identify the problem's hidden convexity, which allows us to show that a variational policy gradient ascent scheme converges to the global optimal policy for general utilities, at a rate of $O(1/t)$, where $t$ is the iteration index. In the special case of cumulative rewards, our result improves upon the best known convergence rate $O(1/\sqrt{t})$ for tabular policy gradient \cite{agarwal2019optimality}, and matches the convergence rate of variants of the algorithm such as softmax policy gradient \cite{mei2020global} and natural policy gradient \cite{agarwal2019optimality}.  In the case where the utility is strongly concave in occupancy measures (e.g., utilities involving Kullback-Leiber divergence), we established the exponential convergence rate of the variational gradient scheme.
	\end{itemize}

	\paragraph{Related Work.} Policy gradient methods have been extensive studied for RL with cumulative returns. There is a large body of work on variants of policy-based methods as well as theoretical convergence analysis for these methods. Due to space constraints, we defer a thorough review to Supplement \ref{supp-relatedwork}.

\paragraph{Notation.} We let $\RR$ denote the set of reals.
We also let $\norm{\cdot}$ denote the $2$-norm, while for matrices we let it denote the spectral norm. 
For the $p$-norms $(1\le p \le \infty$), we use $\norm{\cdot}_p$. For any matrix $B$, $\|B\|_{\infty,2}:=\max_{\|u\|_\infty\leq1} \|Bu\|_2$.
For a differentiable function $f$, we denote by $\nabla f$ its gradient. If $f$ is nondifferentiable, we denote by
 $\hat \partial f $ the Fr\'echet superdifferential of $f$; see e.g. \cite{drusvyatskiy2019efficiency}.

\section{Problem Formulation}

Consider a Markov decision process (MDP) over the finite state space $ \cS$ and a finite action space $\cA$. For each state $i\in \cS$, a transition to state $j\in \cS$ occurs when selecting action $a\in\cA$ according to a conditional probability distribution $j\sim \cP(\cdot | a, i )$, for which we define the short-hand notation $P_{a}(i,j)$. Let $\xi$ be the initial state distribution of the MDP.
We let $S$ denote the number of states and $A$ the number of actions.
The goal is to prescribe actions based on previous states in order to maximize some long term objective. We call $\pi: \cS \rightarrow P(\cA)$ a \emph{policy} that maps states to distributions over actions, which we subsequently stipulate is stationary.  In the standard (cumulative return) MDP, the objective is to maximize the expected cumulative sum of future rewards \cite{puterman2014markov}, i.e.,
\begin{equation}\label{eq:value}
\max_{\pi } V^{\pi}(s):=\mathbb{E}\left[\sum_{t=0}^\infty \gamma^t  r_{s_t  a_t}~\bigg|~ i_0 = s, a_t \sim \pi(\cdot|s_t), t=0,1,\dots\right], \quad\forall s\in\cS.
\end{equation}
with reward $ r_{s_t  a_t}\in\RR$ revealed by the environment when 
action $a_t$ is chosen at state $s_t$.

%



In this paper we consider policy optimization for maximizing general objective functions that are not limited to cumulative rewards. In particular, we consider the problem
\begin{equation}\label{eq-problem}
\max_{\pi} R( \pi) :=  F(\lambda^{\pi})
\end{equation}
where $\lambda^\pi$ is known as the {\it cumulative discounted state-action occupancy measure}, or {\it flux} under policy $\pi$, and $F$ is a general concave functional. Denote $\das$ and $\cL$ as the set of policy and flux respectively, then $\lambda^\pi$ is given by the mapping $\Lambda:\das\mapsto\cL$ as 
\begin{equation}
\label{prop:dual-meaning-3}
\lambda_{sa}^\pi = \Lambda_{sa}(\pi):= \sum_{t=0}^\infty\gamma^t\cdot\mathbb{P}\Big(s_t = s, a_t = a\,\,\Big|\,\,\pi, s_0\sim\xi\Big)\quad\mbox{for}\quad\forall a\in\cA, \forall s\in\cS\,.
\end{equation}

Similar to the LP formulation of a standard MDP, we can write \eqref{eq-problem} equivalently as an optimization problem in $\lambda$
 (see \cite{zhang2020cautious}), giving rise to
\begin{eqnarray}
\label{prob:dual}
\max_\lambda  \F(\lambda)\qquad
\text{s.t.} \qquad\sum\limits_{a\in\cA}(I-\gamma P_a^\top) \lambda_a=\xi,  \lambda\geq0,
\end{eqnarray}
where $\lambda_a = [\lambda_{1a},\cdots,\lambda_{Sa}]^\top\in\RR^{A}$ is the $a$-th column of $\lambda$
and $\xi$ is the initial distribution over the state space $\cS$. 
The constraints require that $\lambda$ be the unnormalized state-action occupancy measure corresponding to {\it some} policy. In fact, it is well known that a policy $\pi$ inducing $\lambda$ can be extracted from $\lambda$ using the mapping $\Pi: \cL\mapsto\das$ as $\pi(a|s) = \Pi_{sa}(\lambda): =  \frac{\lambda_{sa}}{\sum_{a'\in\cA}\lambda_{sa'}}$ for all $a,s$.

Problem \eqref{eq-problem} contains the original MDP problem as a special case. To be specific, when $F(\lambda)=\langle r, \lambda \rangle$ with $r\in\RR^{SA}$ as the reward function, then $F(\lambda)= \langle\lambda,r\rangle = \mathbb{E}\big[\sum_{t=0}^\infty \gamma^t r_{s_t a_t}~\big|~\pi,  s_0 \sim \xi\big]$. This means that \eqref{prob:dual} is a generalization of \eqref{eq:value}, and reduces to the dual LP formulation of standard MDP for this (linear) choice of $F(\cdot)$ \cite{Ka83Thesis}.
We focus on the case where $F$ is concave, which makes \eqref{prob:dual} a concave (hence, convenient) maximization problem.  Next we introduce a few examples that arise in practice for incentivizing safety, exploration, and imitation, respectively.  

\begin{example}[\bf MDP with Constraints or Barriers]\normalfont\label{eg:cmdp}
In discounted constrained MDPs the goal is to maximize the total expected discounted reward 
under a constraint where for some cost function $c:\cS \times \cA \to \mathbb{R}$, the total expected discounted cost incurred by the chosen policy is constrained from above.
Letting $r$ denote the reward function over $\cS\times \cA$, the underlying optimization problem becomes 
\begin{eqnarray}
\label{prob:cmdp}
\max_\pi v^{\pi}_r: = \mathbb{E}^{\pi}\left[\sum^{\infty}_{t=0} \gamma^t r(s_t,a_t) \right]
\qquad \text{s.t.} \qquad
 v^{\pi}_c: = \mathbb{E}^{\pi}\left[\sum^{\infty}_{t=0} \gamma^t c(s_t,a_t) \right] \leq C.
\end{eqnarray}
As is well known, a relaxed formulation is
\begin{eqnarray}
\label{prob:dual-cmdp}
\max_\lambda  \F(\lambda):=\ip{\lambda,r}- \beta\cdot p (\ip{\lambda,c}- C)\qquad
\text{s.t.} \qquad\sum\limits_{a\in\cA}(I-\gamma P_a^\top) \lambda_a=\xi,  \lambda\geq0.
\end{eqnarray}
where $p$ is a penalty function (e.g., the log barrier function).
\end{example}

\begin{example}[\bf Pure Exploration]\normalfont\label{eg:explore}
In the absence of a reward function, an agent may consider the problem of finding a policy whose stationary distribution has the largest ``entropy'', as this should facilitate maximizing the speed at which the agent explores its environment
\cite{hazan2018provably}:
\begin{eqnarray}
\label{prob:exploration-1}
\max_\pi  R(\pi):=    \Entropy(\bar\lambda^{\pi}),
\end{eqnarray}
where $\bar\lambda^{\pi}$ is the normalized state visitation measure given by $\bar\lambda^{\pi}_{s} = (1-\gamma) \sum_a \lambda^{\pi}_{sa}$ for all $s$.
Various entropic measures are possible, but the simplest is the negative log-likelihood: 
$\Entropy(\bar\lambda^{\pi}) = - \sum_{s}\bar{\lambda}_{s}^{\pi}\log[\bar{\lambda}_{s}^{\pi}]$. 
As is well known, this entropy is (strongly) concave. 

Another example, when $d$ state-action features $\boldsymbol{\phi}(s,a)\in\RR^d$ are available, is to cover the entire feature space by maximizing the smallest eigenvalue of the covariance matrix:
\begin{eqnarray}
\label{prob:exploration-2}
\max_\pi  R(\pi):= \sigma_{\min} \left( \mathbb{E}^{\pi} \left[\sum^{\infty}_{t=1} \gamma^t \boldsymbol{\phi}(s_t,a_t)\boldsymbol{\phi}(s_t,a_t)^{\top}\right] \right).
\end{eqnarray}
\end{example}
%
%
%

In \eqref{prob:exploration-2}, observe that $\mathbb{E}^{\pi} [\sum^{\infty}_{t=1} \gamma^t \boldsymbol{\phi}(s_t,a_t)\boldsymbol{\phi}(s_t,a_t)^{\top}]=\sum_{sa}\lambda^\pi_{sa}\cdot\boldsymbol{\phi}(s,a)\boldsymbol{\phi}(s,a)^\top$. By Rayleigh principle, 
it is again a concave function of $\lambda$.

\begin{example}[{\bf Learning to mimic a demonstration}]\label{eg:KL}\normalfont
When demonstrations are available, they may be employed to obtain information about a prior policy in the form of a state visitation distribution $\bar\mu$. Remaining close to this prior can be achieved by minimizing the Kullback-Liebler (KL) divergence between the state marginal distribution of $\lambda$ and the prior $\bar\mu $ stated as 
	\begin{align}\label{risk_not_r}
	F(\lambda) =  \hbox{KL}\Big((1-\gamma)\sum_a{\lambda}_a || \bar\mu \Big)
	\end{align}
	which, when substituted into \eqref{prob:dual}, yields a method for ensuring some baseline performance. 
	We further note that in place of KL divergence, one can also use other convex distances such as Wasserstein, total variation, or Hellinger distances.
\end{example}

Additional instances may be found in \cite{zhang2020cautious}. With the setting clarified, we shift focus to developing an algorithmic solution to \eqref{prob:dual}, that is, to solve for policy $\pi$.


\section{Variational Policy Gradient Theorem}
\label{sec:pg-estimate} 
To handle the curse of dimensionality, we allow parametrization of the policy by $\pi=\pi_\theta$, where  $\theta\in\Theta\subset\RR^d$ is the parameter vector. In this way, we can narrow down the policy search problem to within a $d$-dimensional parameter space rather than the high-dimensional space of tabular policies.  The policy optimization problem then becomes
\begin{equation}
\label{prob:dual-para-pi}
\max_{\theta\in\Theta} \,\, R(\pi_\theta):= \F(\lambda^{\pi_\theta})
\end{equation}
where $F$ is the concave utility of the state-action occupancy measure $\lambda(\theta):=\lambda^{\pi_\theta}$ , $\Theta\subset\mathbb{R}^d$ is a convex set. 
 We seek to solve for the policy maximizing the utility as in \eqref{prob:dual-para-pi} using gradient ascent over the parameter space $\Theta$. Note that \eqref{prob:dual-para-pi} is simply \eqref{eq-problem} with parameterization $\theta$ of policy $\pi$ substituted.
 We denote by $\nabla_{\theta} R(\pi_{\theta})$ the parameterized policy gradient of general utility. 
  
First, recall the policy gradient theorem for RL with cumulative rewards \cite{sutton2000policy}. Let the reward function be $r$. Define $V(\theta;r):=\ip{\lambda(\theta),r}$, i.e., the total expected discounted reward under the reward function $r$ and the policy $\pi_\theta$. 
The Policy Gradient Theorem states that 
%
\begin{equation}\label{eq:vanilla_pg}
\nabla_\theta V(\theta;r) =\EE^{\pi_\theta}\left[\sum^{\infty}_{t=0}\gamma^t Q^{\pi_\theta}(s_t,a_t;r)\cdot\nabla_\theta \log \pi_\theta(a_t|s_t)\right],
\end{equation}
where $Q^\pi(s,a;r):=\EE^\pi\big[\sum_t \gamma^t r(s_t,a_t) \mid s_0=s, a_0=a, a_t\sim \pi(\cdot \mid s_t)\big] $. Unfortunately, this elegant result no longer holds when we consider a general function instead of cumulative rewards:
The policy gradient theorem relies on the additivity of rewards, which is lost in our problem. 
For future reference, we denote $ Q^\pi(s,a;z):=\EE^\pi\big[\sum_t \gamma^t z_{s_t a_t} \mid s_0=s, a_0=a, a_t\sim \pi(\cdot \mid s_t)\big] $ where $z$ is any ``function'' of the state-action pairs ($z\in\RR^{SA}$).
Moreover, $V(\theta;z)$ is defined similarly. These definitions are motivated by subsequent efforts to derive an expression for the gradient of \eqref{prob:dual-para-pi}.

\subsection{Policy Gradient of $R(\pi_\theta)$}

Now we derive the policy gradient of $R(\pi_\theta)$ with respect to $\theta$. By the chain rule, the gradient of $\F(\lambda(\theta)):=\F(\lambda^{\pi_\theta})$, using the definition of $R(\pi_{\theta})$, yields (assuming differentiability of $F,\lambda$):
\begin{equation}
	\label{chain-rule}
	\nabla_\theta R(\pi_{\theta}) = \sum_{s\in\cS}\sum_{a\in\cA} \frac{\partial \F(\lambda(\theta))}{\partial \lambda_{sa}}\cdot\nabla_\theta\lambda_{sa}(\theta).
\end{equation}
To directly use the chain rule, one needs the partial derivatives $\frac{\partial \F(\lambda(\theta))}{\partial \lambda_{sa}}$ and $\nabla_\theta\lambda_{sa}(\theta)$. Unfortunately, neither of them is easy to estimate. The partial gradient $\frac{\partial \F(\lambda(\theta))}{\partial \lambda_{sa}}$ is a function of the current state-action occupancy measure $\lambda^{\pi_\theta}$. 
One might attempt to estimate the measure $\lambda^{\pi_\theta}$ and then evaluate the gradient map $\frac{\partial \F(\lambda(\theta))}{\partial \lambda_{sa}}$. However, 
estimates of distributions over large spaces tend to converge very slowly \cite{tsybakov2008introduction}.



As it turns out, a viable alternate route is to consider the Fenchel dual $F^*$ of $F$.
Recall that $F^*(z) = \inf_{\lambda} \ip{\lambda,z} - F(\lambda)$, where we use $\ip{x,y}:=x^\top y$ (since $F$ is concave, the dual is defined using $\inf$, instead of $\sup$).
As is well known, for $F$ concave, under mild regularity conditions, 
the bidual (dual of the dual) of $F$ is equal to $F$.
This forms the basis of our first result, which states 
that the steepest policy ascent direction of \eqref{prob:dual-para-pi} 
is the solution to a stochastic saddle point problem. The proofs of this and subsequent results are given in the supplementary material.

\begin{theorem}[{\bf Variational Policy Gradient Theorem}\label{thm-PG}] Suppose $F$ is concave and continuously differentiable in an open neighborhood of $\lambda^{\pi_{\theta}}$. Denote $V(\theta;z)$ to be the cumulative value of policy $\pi_{\theta}$ when the reward function is $z$, and assume $\nabla_{\theta} V(\theta;z)$ always exists. 
Then we have
\begin{equation}
\label{thm-PG-2}
{\nabla_{\theta} R(\pi_{\theta}) }  = \lim_{\delta\to0_+} \argmax_{x} \inf_{z} \left\{ V(\theta;z) + \delta\nabla_{\theta} V(\theta;z)^{\top} x- F^*(z) -  \frac{\delta}{2}\|x\|^2 \right\}.
\end{equation}

\end{theorem}
Therefore, to estimate $\nabla_{\theta} R(\pi_{\theta})  $ we require the cumulative return $V(\theta;z)$ of the function $z$, its associated ``vanilla" policy gradient \eqref{eq:vanilla_pg}, and the gradient of the Fenchel dual of $F$ at $z$. 
These ingredients are combined via \eqref{thm-PG-2} to obtain a valid policy gradient for general objectives. Next, we discuss how to estimate the gradient using sampled trajectories.



\subsection{Estimating the Variational Policy Gradient}

Theorem \ref{thm-PG} implies that one can estimate $\nabla_{\theta} R(\pi_{\theta}) $ by solving a stochastic saddle point problem. Suppose we generate $n$ i.i.d. episodes  of length $K$ following $\pi_{\theta}$, denoted as $\zeta_i = \{s_k^{(i)},a_k^{(i)}\}_{k=1}^K$. Then we can estimate $V(\theta;z)$ and $\nabla V(\theta;z)$ for any function $z$ by\vspace{-2mm}
\begin{align}\label{eq:empirical_value_estimate}
\tilde V(\theta;z)&:= \frac{1}{n}\sum_{i=1}^n  V(\theta;z;\zeta_i)
:= \frac{1}{n}\sum_{i=1}^n \sum^K_{k=1} \gamma^k\cdot z(s_k^{(i)},a_k^{(i)}),\\
\nabla \tilde V(\theta;z)&:= \frac{1}{n}\sum_{i=1}\nabla_{\theta}V(\theta;z;\zeta_i)
:= \frac{1}{n}\sum_{i=1}^n\sum^K_{k=1}\sum_{a\in\cA} \gamma^k\cdot Q(s_k^{(i)},a; z)\nabla_{\theta}\pi_{\theta}(a|s_k^{(i)}). \nonumber
\end{align} 
For a given value of $K$, the error introduced by ``truncating'' trajectories at length $K$ is of order $\gamma^K/(1-\gamma)$, which quickly decays to zero for $\gamma<1$.
Plugging in the obtained estimates into \eqref{thm-PG-2} gives rise to the 
sample-average approximation to the policy gradient:
\begin{equation}\label{eq-esp}
\hat \nabla_{\theta} R(\pi_{\theta};\delta) := \argmax_{x} \inf_{\|z\|_\infty\leq\ell_F} \left\{-F^*(z) + \tilde V(\theta;z) + \delta \nabla_{\theta}\tilde V(\theta;z)^\top x - \frac{\delta}{2}\|x\|^2\right\},
\end{equation}
where $\ell_F$ is defined in the next theorem. Therefore, any algorithm that solves problem \eqref{eq-esp} will serve our purpose. A MC stochastic approximation scheme, i.e., Algorithm \ref{alg:MCDPG}, is provided in Appendix \ref{appdx:alg-PG-estimation}. 

\begin{theorem}[{\bf Error bound of policy gradient estimates}\label{thm-PGestimate}] Suppose the following holds:\\
\textbf{(i)} $\text{dom} F = \RR^{SA}$, there exists $ ~\ell_F$ such that $\max\{\|\nabla F(\lambda)\|_\infty: \|\lambda\|_1\leq \frac{2}{1-\gamma}\}\leq \ell_F$. \\
\textbf{(ii)} $F$ is $L_F$-smooth under $L_1$ norm, i.e., $\|\nabla F(\lambda) - \nabla F(\lambda')\|_\infty\leq L_F\|\lambda-\lambda'\|_1$.\\
\textbf{(iii)} $F^*$ is $(\ell_{F^*})$-Lipschitz  with respect to the $L_\infty$ norm in the set $\{z:\|z\|_\infty\leq 2\ell_F,\,F^*(z)>-\infty\}$.\\
 \textbf{(iv)} There exists $C$ with $\|\nabla_{\theta}\pi(\cdot|s)\|_{\infty,2}\leq C$, where $\nabla_{\theta}\pi(\cdot|s) = \left[\nabla_{\theta}\pi(1|s),\cdots,\nabla_{\theta}\pi(A|s)\right]$.  \\ 
 Let $\hat \nabla_{\theta} R(\pi_{\theta}) := \lim_{\delta\to0_+}\hat \nabla_{\theta} R(\pi_{\theta};\delta)$.
Then
$$\EE[\|\hat \nabla_{\theta} R(\pi_{\theta})  - \nabla_{\theta} R(\pi_{\theta})\|^2]\leq \cO\left(\frac{C^2(\ell_F^2 + L_F^2\ell_{F^*}^2)}{n(1-\gamma)^4} + \frac{C^2L_F^2}{n(1-\gamma)^6}\right) + \cO(\gamma^K).$$\vspace{-2mm}
\end{theorem}

{\bf Remarks.}  \\
{\bf (1)} Theorem \ref{thm-PGestimate} suggests an $O(1/\sqrt{n})$ error rate, proving that the variational policy gradient - though more complicated than the typical policy gradient that takes the form of a mean - can be efficiently estimated from finite data. \\
{\bf (2)} Although the variable $z$ is high dimensional, our error bound depends only on the properties of $F$.\\
 {\bf (3)} We assumed for simplicity that Q values are known. In practice, they can estimated by, e.g., an additional Monte Carlo rollout on the same sample path or temporal difference learning. As long as the estimator for $Q(s,a;z)$ is unbiased and upper bouded by $\cO(\frac{\|z\|_\infty}{1-\gamma})$, the result  will not change.\\
  {\bf (4)} For the case of cumulative rewards, we have $F(\lambda) = \langle r,\lambda\rangle$, so that  
  $\ell_F$ $=$ $\|r\|_\infty$, $\,\ell_{F^*}$$=$$0, L_F$$=$$0$. Therefore {\footnotesize$\EE[\|\hat \nabla_{\theta} R(\pi_{\theta})  - \nabla_{\theta} R(\pi_{\theta})\|^2]\leq \cO\left(\frac{C^2\|r\|_\infty^2}{n(1-\gamma)^4}\right).$}



{\bf Special cases of $\nabla_{\theta}R(\pi^{\theta})$.} We further explain how to obtain the variational policy gradient for several special cases of $R$, including constrained MDP, maximal exploration, and learning from demonstrations. See Appendix \ref{appdx:special_cases} for more details. 

\section{Global Convergence of Policy Gradient Ascent}\label{sec:convergence}

In this section, we analyze policy search for the problem \eqref{prob:dual-para-pi}, i.e., $\max_{\theta\in\Theta} R(\pi_{\theta})$ via gradient ascent:
\begin{eqnarray}
\label{defn:grad-proj}
\theta^{k+1}  =     \argmax_{\theta\in\Theta} R(\pi_{\theta^k}) \!+\! \left\langle \nabla_\theta R(\pi_{\theta^k}),\theta\!-\!\theta^k\right\rangle \!-\! \frac{1}{2\eta}\|\theta\!-\!\theta^k\|^2 =  \Proj_{\Theta}\big\{\theta^k + \eta\nabla_\theta R(\pi_{\theta^k})\big\}
\end{eqnarray}
where $\Proj_{\Theta}\{\cdot\}$ denotes Euclidean projection onto $\Theta$, and equivalence holds by the convexity of $\Theta$.

\subsection{No spurious first-order stationary solutions.}

We study the geometry of the (possibly) nonconvex optimization problem \eqref{prob:dual-para-pi}.  When $F$ is a linear function of  $\lambda$, and the parameterization is tabular or softmax, existing theory of cumulative-return RL problems have shown that every first-order stationary point of \eqref{prob:dual-para-pi} is globally optimal -- see \cite{agarwal2019optimality,mei2020global}.

In what follows, we show that the problem \eqref{prob:dual-para-pi} has no spurious extrema despite of its nonconvexity, for general utility functions and policy parametrization. {Specifically, to generalize global optimality attributes of stationary points of \eqref{prob:dual-para-pi} from \eqref{eq:value}, we exploit structural aspects of the relationship between occupancy measures and parameterized families of policies, namely, that these entities are related through a bijection. This bijection, when combined with the fact that \eqref{prob:dual-para-pi} is concave in $\lambda$, and suitably restricting the parameterized family of policies, is what we subsequently describe as ``hidden convexity."} For these results to be valid, we require the following regularity conditions.
\begin{assumption}
	\label{assumption:gen-para}
	Suppose the following holds true:\\
	\textbf{(i).} $\lambda(\cdot)$ forms a bijection between $\Theta$ and $\lambda(\Theta)$, where $\Theta$ and $\lambda(\Theta)$ are closed and convex. \\
	\textbf{(ii).} The Jacobian matrix $\nabla_{\theta}\lambda(\theta)$ is Lipschitz continuous in $\Theta$.  \\
	\textbf{(iii).} Denote $g(\cdot) := \lambda^{-1}(\cdot)$ as the inverse mapping of $\lambda(\cdot)$. Then there exists $ \ell_{\theta}>0$ s.t. $\|g(\lambda)-g(\lambda')|\leq \ell_\theta\normm{\lambda-\lambda'}$ for some norm $\normm{\cdot}$ and for all $\lambda,\lambda'\in\lambda(\Theta)$. 
\end{assumption}
In particular, for the direct policy parametrization, also known as the ``tabular" policy case, we have $\lambda(\theta) := \Lambda(\pi)$ where $\Lambda$ is defined in \eqref{prop:dual-meaning-3}. When $\xi$ is positive-valued, Assumption \ref{assumption:gen-para} is true for the tabular policy case (as established in Appendix \ref{appdx:bijection}).
%
\begin{theorem}[{\bf Global optimality of stationary policies}]
	\label{theorem:global-opt}
	Suppose Assumption \ref{assumption:gen-para} holds, and $\F$ is a concave, and continuous function defined in an open neighbourhood containing $\lambda(\Theta)$. 
	Let $\theta^*$ be a first-order stationary point of problem \eqref{prob:dual-para-pi}, i.e.,\vspace{-2mm}
	\begin{equation}
	\label{defn:1st-order-condition}
	\exists u^*\in\hat{\partial}(\F\circ\lambda)(\theta^*),\quad\text{s.t.}\quad \langle u^*, \theta-\theta^*\rangle\leq0 \qquad\mbox{for}\qquad\forall \theta\in\Theta.
	\end{equation}
	Then $\theta^*$ is a globally optimal solution of problem \eqref{prob:dual-para-pi}.	
\end{theorem}
%
%
Theorem \ref{theorem:global-opt} provides conditions such that, despite of nonconvexity, local search methods can find the global optimal policies. 	Since we aim at general utilities, we naturally separated out the convex and non-convex maps in the composite objective
	and our conditions for optimality rely on the properties of these.
In a recent paper, 
	\citet{BhaRu20} proposed some sufficient conditions under which a result similar to Theorem~\ref{theorem:global-opt} holds in the setting of the standard, cumulative total reward criterion. Their conditions are {\em (i)} the policy class is closed under (one-step, weighted) policy improvement
	and that {\em (ii)} all stationary points of the one-step policy improvement map are global optima of this map. 
	It remains for future work to see the relationship between our conditions and these conditions: They appear to have rather different natures.



\subsection{Convergence analysis} 
Now we analyze the convergence rate of the policy gradient scheme \eqref{defn:grad-proj} for general utilities. 
\begin{assumption}
	\label{assumption:ncvx-Lip}
	There exists $ L>0$ such that the policy gradient $\nabla_{\theta }R(\pi_{\theta})$ is $L$-Lipschitz. 
\end{assumption}


The objective $R(\pi_{\theta})$ is nonconvex in $\theta$, so one might expect that gradient schemes converge to stationary solutions at a standard $\cO(1/\sqrt{t})$ convergence rate \cite{shapiro2014lectures}. 
Remarkably, the policy optimization problem admits a convex nature if we view it in the space of $\lambda$, as long as $F$ is concave.
By exploiting this hidden convexity, we establish an $\cO(1/t)$ convergence rate for solving RL with general utilities. Further, we show that, when the utility $F$ is strongly concave, the gradient ascent scheme converges to the globally optimal policy exponentially fast. 
\begin{theorem}[{\bf Convergence rate of parameterized policy gradient iteration}]
	\label{theorem:iteration complexity-gen}
	Let Assumptions \ref{assumption:gen-para} and \ref{assumption:ncvx-Lip} hold. Denote $D_\lambda \!:=\! \max_{\lambda,\lambda'\in\lambda(\Theta)} \normm{\lambda-\lambda'}$ as defined in Assumption \ref{assumption:gen-para}(iii). Then the policy gradient update \eqref{defn:grad-proj} with $\eta = 1/L$ satisfies for all $k$\vspace{-2mm}
	\begin{equation*}
	R(\pi_{\theta^*}) \!-\! R(\pi_{\theta^k})\leq \frac{4L\ell_{\theta}^2D_\lambda^2}{k+1}.
	\end{equation*}
	Additionally, if $F(\cdot)$ is $\mu$-strongly concave with respect to the $\normm{\cdot}$ norm, we have \vspace{-2mm}
	\begin{equation*}
	R(\pi_{\theta^*})\!-\!R(\pi_{\theta^k})\!\leq\! \Big(1\!-\!\frac{1}{1\!+\!L\ell^2_\theta/\mu}\Big)^k\!\left(R(\pi_{\theta^*})\!-\!R(\pi_{\theta^0})\right).
	\end{equation*}
\end{theorem} 
 The exponential convergence result of Theorem \ref{theorem:iteration complexity-gen} implies that, when a regularizer like Kullback-Leiber divergence is used, policy gradient method converges much faster. In other words, policy search with general utilities can actually be easier than the typical, cumulative-return problem.

Finally, we study the case where policies are not parameterized, i.e., $\theta=\pi$. The next theorem establishes a tighter convergence rate than what Theorem \ref{theorem:iteration complexity-gen} already implies.



\begin{theorem}[{\bf Convergence rate of tabular policy gradient iteration}]
\label{theorem:iteration complexity}
Let $\theta = \pi$ and $\lambda(\theta) = \Lambda(\pi)$. Let Assumption \ref{assumption:ncvx-Lip} hold and assume that $\xi$ is positive-valued. Then 
the iterates generated by \eqref{defn:grad-proj} with $\eta = 1/L$ satisfy for all $k\geq1$ that \vspace{-2mm}
\begin{equation*}
R(\pi^*) - R(\pi^k)\leq \frac{20L|\cS|}{(1-\gamma)^2(k+1)}\cdot\Big\|d^{\pi^*}_\xi/\xi\Big\|_\infty^2.
\end{equation*}
\end{theorem}

\paragraph{The case of cumulative rewards.} Let us consider the well-studied special case where $F$ is a linear functional, i.e., $R(\pi) \!=\! V^\pi$ [cf. \eqref{eq:value}] is the typical cumulative return. In this case, we have $L \!=\! \frac{2\gamma A}{(1-\gamma)^3}$ (\cite{agarwal2019optimality}). Now in order to obtain an $\epsilon$-optimal policy $\bar{\pi}$ such that $V^{\pi^*}\!\!\! -\! V^{\bar\pi}\leq \epsilon$, the gradient ascent update requires {\footnotesize$\cO\Big(\frac{SA }{(1-\gamma)^5\epsilon}\cdot\big\|d^{\pi^*}_\xi/\xi\big\|_\infty^2\Big)$} iterations according to Theorem \ref{theorem:iteration complexity}. 
This bound is strictly smaller than the {\footnotesize$\cO\Big(\frac{SA }{(1-\gamma)^6\epsilon^2}\big\|d^{\pi^*}_\xi/\xi \big\|_\infty^2 \Big)$} iteration complexity proved by \cite{agarwal2019optimality} for tabular policy gradient. The improvement from $O(1/\epsilon^2)$ to $(1/\epsilon)$ comes from the fact that, although the policy optimization problem is nonconvex, our analysis exploits its hidden convexity in the space of $\lambda$.

\section{Experiments}\label{sec:experiments}

\vspace{-2mm}
Now we shift to numerically validating our methods and theory on OpenAI Frozen Lake \cite{brockman2016openai}. Throughout, additional details may be found in Appendix \ref{apx_experiments}.


\begin{figure}[h!]
  \hspace{-2mm}
\centering
\begin{minipage}[t]{.33\textwidth}
  \centering
  \hspace{-1cm}
		{\includegraphics[scale=0.25]{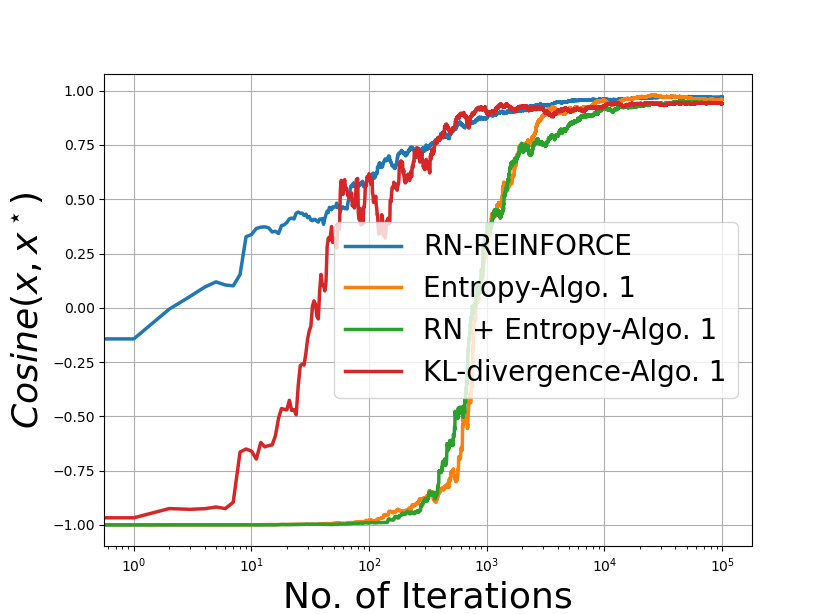}\label{fig:gradient_Estimates}}
		 \hspace{-1.1cm}\vspace{4mm}
		\caption{\scriptsize {\bf PG estimation via Alg. \ref{alg:MCDPG}}
		Cosine similarity between PG estimates $\hat{x}_t $ generated by Algorithm \ref{alg:MCDPG} after $t$ samples and the ground truth ${x}^\star$, which consistently converges to near $1$ across different instances (E.g. \eqref{eg:cmdp} - \eqref{eg:KL}) when $t$ becomes large. For comparison, we also include the convergence of PG estimates from REINFORCE for cumulative returns. }
  \label{fig:gradient_estimate}
\end{minipage}
%
  \hspace{1mm}
\begin{minipage}{.63\textwidth}
\centering
\vspace{0.2cm}
\hspace{-1.1cm}
		\subfigure[Entropy vs.  \# episodes ]{\includegraphics[scale=0.25]{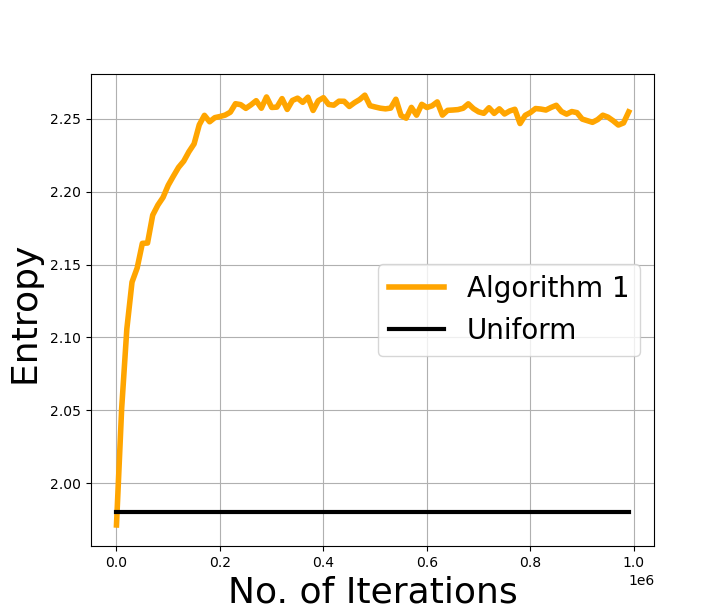}\label{fig:entropy_cost}}
\hspace{-.2cm}
		\subfigure[\scriptsize{World \& occupancy dist. (Entropy)}	]{\includegraphics[scale=0.2]{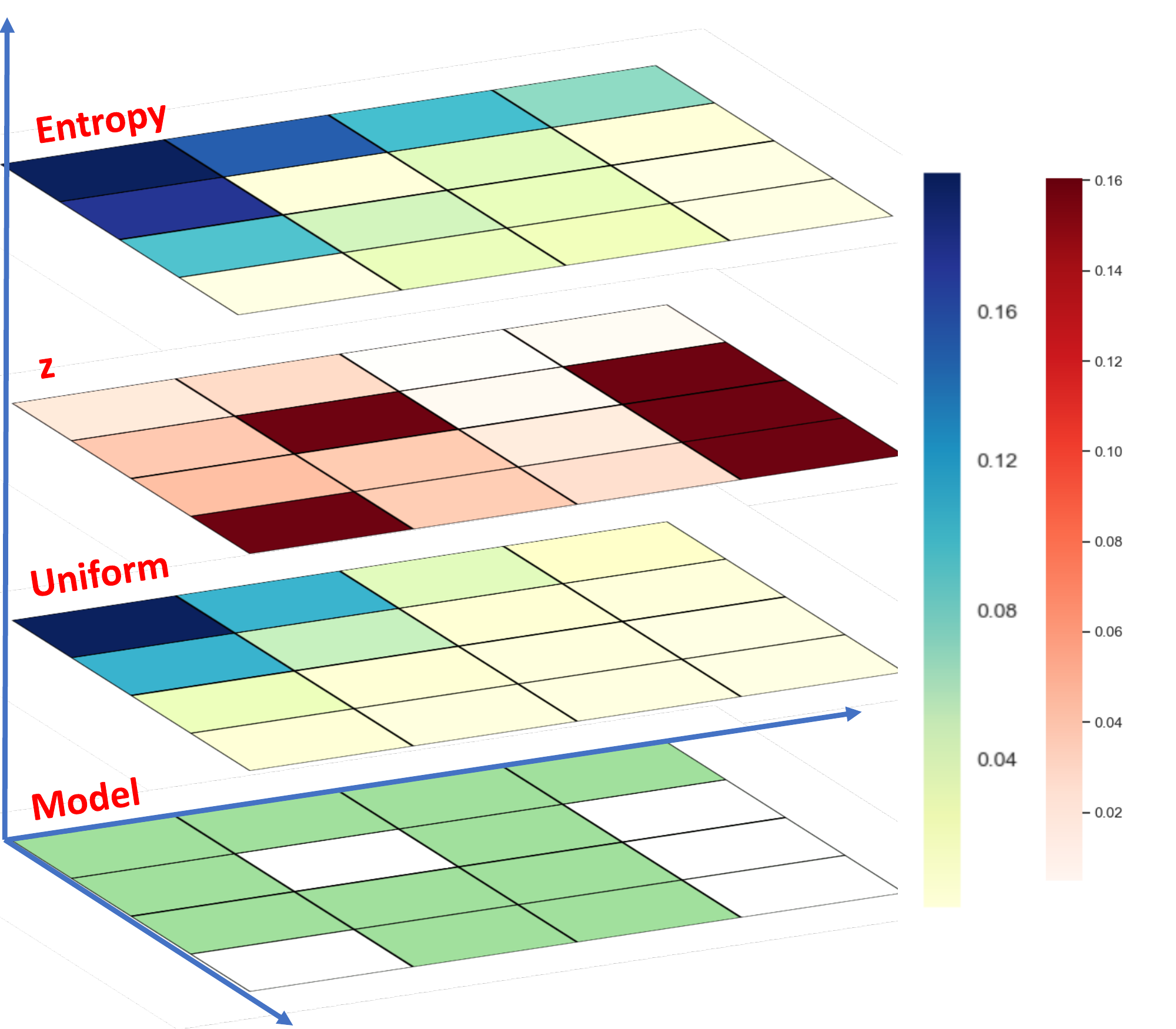}\label{fig:entropy_occupancy}}
		\hspace{-1cm}\vspace{-2mm}
		\caption{\scriptsize  {\bf Results for maximum entropy exploration}:  In Fig. \ref{fig:entropy_cost}, to quantify exploration, we present the entropy of flux $\lambda$ over training index $n$ for our approach, as compared with the entropy of a uniform random policy. Fig. \ref{fig:entropy_occupancy}(bottom) visualizes the world model (holes in the lake have null entropy, as they terminate the episode), the lower middle layer displays the occupancy measure associated with a uniformly random policy, the upper-middle visualizes the pseudo-reward $z^*$ defined by the Fenchel dual of the entropy \eqref{prob:exploration-1} -- see Appendix \ref{appdx:special_cases}. Lastly, on top we visualize the occupancy measure associated with the max entropy policy, which better covers the space than a uniformly random policy. 
		}\label{fig:entropy}
  \label{fig:test2}
\end{minipage}\vspace{-5mm}
\end{figure}
%

\noindent {\bf Policy Gradient (PG) Estimation.}  First we investigate the use of Theorem \ref{thm-PG} and Algorithm \ref{alg:MCDPG} (Appendix \ref{appdx:alg-PG-estimation}) for PG estimation, for several instances of the general utility. We also compare it with the gradient estimates computed by REINFORCE for cumulative returns. 
Specifically, in Figure \ref{fig:gradient_estimate} we illustrate the convergence of gradient estimates, measured using the cosine similarity between $x_n$ (running estimate based on $n$ episodes) and the true gradient $x^*$ (which is evaluated using brute force Monte Carlo rollouts -- see Appendix \ref{apx_experiments_benchmark}). The cosine similarity converges to $1$ across different instances, providing evidence that Algorithm \ref{alg:MCDPG} yields consistent gradient estimates for general utilities. 

\noindent {\bf PG Ascent for Maximal Entropy Exploration.}  
Next, we consider maximum entropy exploration \eqref{prob:exploration-1} using algorithm \eqref{defn:grad-proj}, with softmax parametrization. First, we display the evolution of the entropy of the normalized occupancy measure over the number of episodes in Fig. \ref{fig:entropy_cost}. Then, we visualize the world model in Fig. \ref{fig:entropy_occupancy}(bottom). Moreover, the lower middle is the occupancy measure associated with a uniformly random policy, the upper-middle layer visualizes the "pseudo-reward" $z^*$ computed as the Fenchel dual of the entropy \eqref{prob:exploration-1} -- see Appendix \ref{appdx:special_cases}, which is null at the holes and positive otherwise. We use a different color to denote that its values are not likelihoods. The occupancy measure obtained by policy gradient ascent with gradient estimated by Algorithm \ref{alg:MCDPG}  at the end of training is in Figure \ref{fig:entropy_occupancy}(top) -- observe the maximal entropy policy achieves significantly better coverage of the state space than the uniformly random policy. 

\begin{figure}[h]
\centering
\vspace{-.5cm}
	\hspace{-1cm}
	\subfigure[\scriptsize{World \& occupancy dist. (CMDP)}]{\includegraphics[scale=.25]{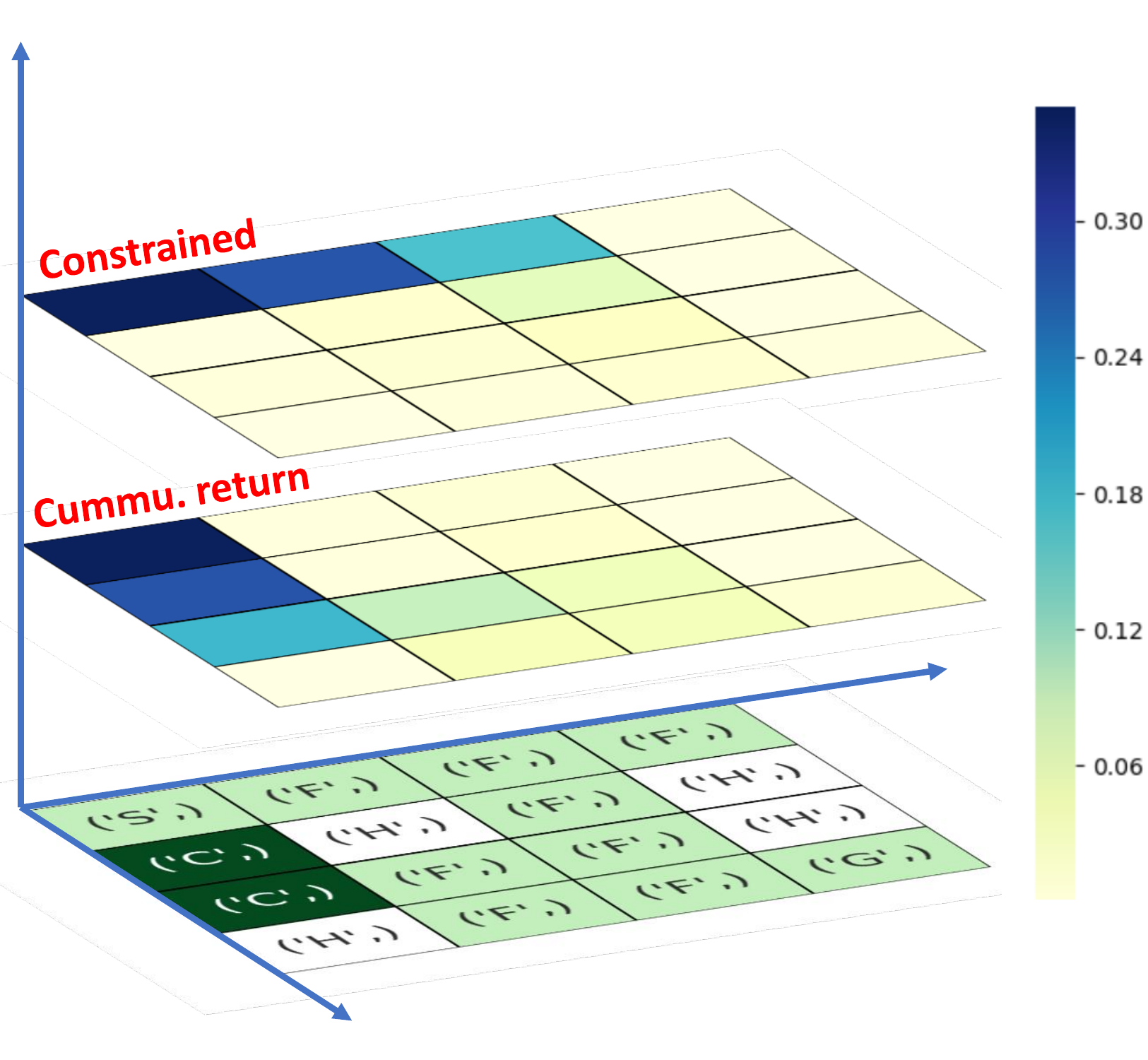}\label{fig:cmdp_occupancy}}\hspace{-1mm}
			\subfigure[Reward vs. \# episodes]{\includegraphics[scale=0.25]{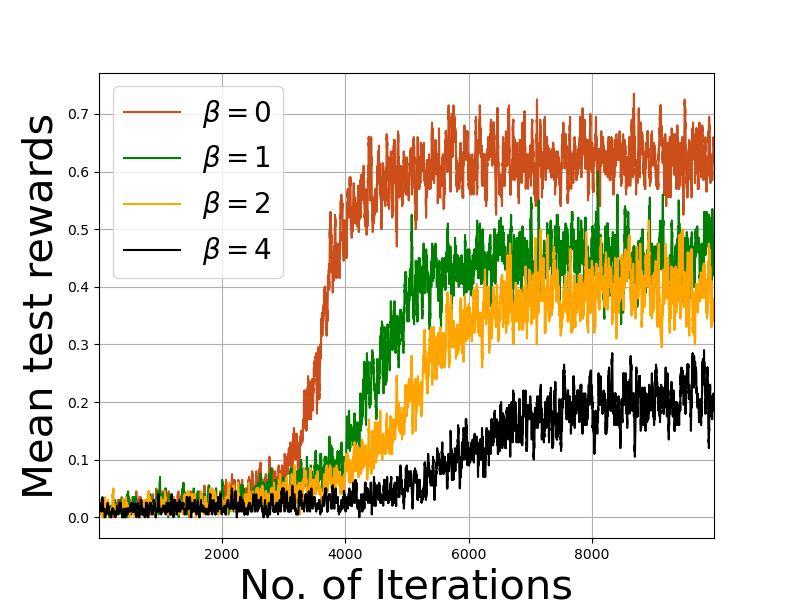}\label{fig:CMDP_value}}\hspace{-6mm}
		\subfigure[Cost  vs.  \# episodes ]{\includegraphics[scale=0.25]{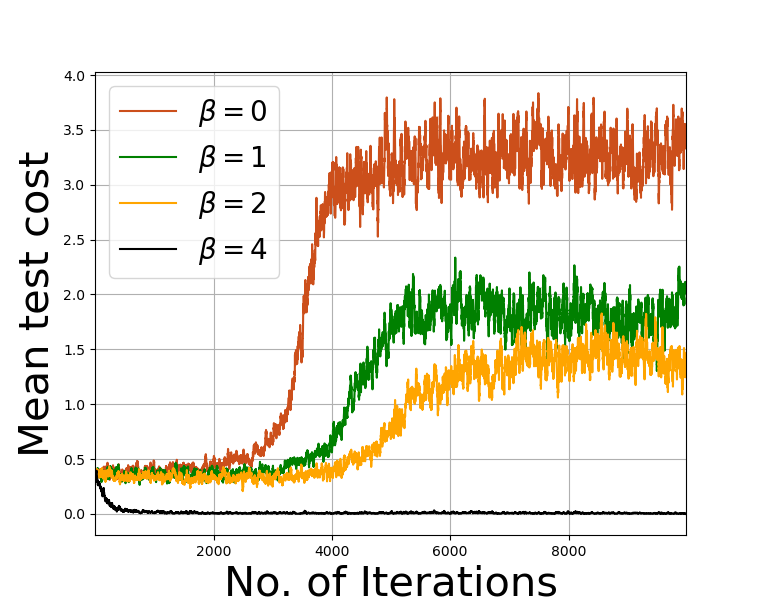}\label{fig:CMDP_cost}}\hspace{-1.4cm}
		\vspace{-.25cm}
				\caption{\scriptsize {\bf Results for avoiding obstacles.} Fig. \ref{fig:cmdp_occupancy}(bottom) depicts the world model of OpenAI Frozen Lake with augmentation to include costly states, e.g., obstacles:  C represents costly states, F is the frozen lake, H is the hole, and G is the goal. We consider softmax policy parameterization, and visualize the  occupancy measure associated with REINFORCE for the cumulative return \eqref{eq:value} in the middle layer, and the relaxed {\bf CMDP} \eqref{prob:dual-cmdp} via a {\bf logarithmic barrier} \eqref{eq:logarithmic_barrier} at the top.\medskip
				 The policy obtained via barriers avoids visiting costly states, in contrast to the middle.
				%
				Fig. \ref{fig:CMDP_value} and Fig. \ref{fig:CMDP_cost} show the reward/cost accumulated during test trajectories over training index for Algorithm \ref{alg:MCDPG}. Observe that the reward/cost curves behave differently as the penalty parameter $\beta$ varies: observe that without any constraint imposition (which implies $\beta=0$ in red), one achieves the highest reward, but incurs the most costs, i.e., hits obstacles most often. Larger $\beta$ imposes more penalty, and hence $\beta=4$ incurs lowest cost and lowest reward. Other instances are also shown for $\beta=1$ and $\beta=2$. \vspace{-3mm}
				}
		 	\label{fig:cmdp_convergence}
\end{figure}

\noindent {\bf PG Ascent for Avoiding Obstacles.}  
Suppose our goal is to navigate the Frozen Lake and avoid obstacles. 
We consider imposing penalties to avoid costly states [cf. \eqref{prob:dual-cmdp}] via a logarithmic barrier \eqref{eq:logarithmic_barrier}, and by applying variational PG ascent, we obtain an optimal policy whose resulting occupancy measure is depicted in Fig. \ref{fig:cmdp_occupancy}(top). 
For comparison, we consider optimizing the standard expected cumulative return \eqref{eq:value}, whose state occupancy measure is given in Fig. \ref{fig:cmdp_occupancy}(middle).
Observe that imposing log penalties yields policies whose probability mass is concentrated away from obstacles (dark green).
Further, we display in Fig. \ref{fig:cmdp_convergence} the reward \ref{fig:CMDP_value} and cost \ref{fig:CMDP_cost} accumulation during test trajectories as a function of the iteration index for the PG ascent \eqref{defn:grad-proj} for the cumulative return \eqref{eq:value} as compared with a logarithmic barrier imposed to solve \eqref{prob:dual-cmdp} for different penalty parameters $\beta$. 

\section{Broader Impact}
While RL has a great number of potential applications, our work is of foundational nature and as such, the application of the ideas in this paper can have both broad positive and negative impacts. However, this paper is purely theoretical, as we do not aim at any specific application, there is nothing we can say about the most likely broader impact of this work that would go beyond speculation.

	\bibliographystyle{plainnat}
	\bibliography{bibliography}
%
	\newpage
	\appendix
\clearpage\newpage\onecolumn
\appendix
\section*{\centering Supplementary Material for \\``Variational Policy Gradient Method\\ for Reinforcement Learning with General Utilities"}

\section{Related Work}
\label{supp-relatedwork}
We provide a more extension discussion for the context of this work. Firstly, when closed-form expressions for the optimizer of a function are unavailable, solving optimization problems requires iterative schemes such as gradient ascent \cite{nocedal2006numerical}. Their convergence to global extrema is predicated on concavity and the tractability of computing ascent directions. When the objective takes the form of an expected value of a function parameterized by a random variable, stochastic approximations are required \cite{robbins1951stochastic,kiefer1952stochastic}. The PG Theorem mentioned above gives a specific form for obtaining ascent directions with respect to a parameterized family of stationary policies via trajectories in a Markov decision process, when the objective is the expected cumulative return \cite{sutton2000policy}, which gives rise to the REINFORCE algorithm. 
	
	The convergence of policy search for the expected cumulative return has been studied extensively in recent years. Under general parameterizations the problem becomes nonconvex. Hence, early work focused on asymptotic convergence to stationarity \cite{pirotta2015policy} by invoking dynamical systems \cite{borkar2008stochastic}. In actor-critic \cite{konda1999actor,konda2000actor}, one replaces the Monte Carlo rollout of the Q function with a temporal difference estimator \cite{sutton1988learning}, and its asymptotic stability follows similar logic \cite{bhatnagar2009natural}. Another line of work focused on only on per-step value increase, i.e., policy improvement bounds \cite{pirotta2013adaptive,pirotta2015policy}. Recent interest has been on structural results that yield convergence to global optimality: when state transitions are linear \cite{fazel2018global,bu2019lqr}), the policy parameterization is direct (tabular) \cite{BhaRu19,agarwal2019optimality}, function approximation error can be quantified \cite{KaLa02,liu2019neural}. Clever step-size rules have also been designed to ensure convergence to second-order stationary points under general settings \cite{Zhang_preprint}. 
	
These results, however, are restricted to the expected cumulative return, a linear functional of the state-action occupancy measure, and hence do not apply to general concave functionals of the form considered in this work. Early works in operations research consider nonstandard utilities \cite{HuKa94}, motivated by certain variance-penalizations which may also be written as concave functionals of occupancy measures \cite{FiKaLe89}. Similar in spirit to this work is \cite{Ka94I}, as it also puts occupancy measures at the center of its conceptual development. These works develop dynamic programming approaches for tabular settings, and hence are not scalable to problems with large spaces. More recently, maximizing the entropy of the state visitation distribution has been considered \cite{hazan2018provably}, a special case of the concave utilities we study. Moreover, the authors develop a model-based iteratively policy update, which requires explicit knowledge of the transition probability matrix. By contrast, in this work we prioritize model-free approaches for possibly large spaces via the fusion of direct policy search and parameterization over a family of policies.

\section{Supplementary materials of Section \ref{sec:pg-estimate}}
\subsection{A Monte Carlo Algorithm for solving \eqref{eq-esp}}
\label{appdx:alg-PG-estimation}
Note that any algorithm that solves problem \eqref{eq-esp} will serve our purpose. Therefore, we provide a Monte Carlo method that alternates between stochastic primal and dual updates as an example, stated in Algorithm \ref{alg-PG-estimation}, in which the projection operator onto the set $\{z:\|z\|_\infty\leq \ell_F\}$ is denoted as $\Proj_{\ell_F}\{z\}$. For any $z$, $z' = \Proj_{\ell_F}\{z\}$ is defined as 
$$z'_i = \begin{cases}
-\ell_F, & \mbox{if } z_i\in(-\infty,-\ell_F),\\
z_i, & \mbox{if } z_i\in[-\ell_F,\ell_F],\\
\ell_F, & \mbox{if } z_i\in(\ell_F,+\infty).
\end{cases}$$
\begin{algorithm}[h]
	\caption{Monte Carlo Variational Policy Gradient Estimation \label{alg-PG-estimation}}
	\begin{algorithmic}
		\label{alg:MCDPG}
		\REQUIRE a differentiable policy parametrization $\pi_{\boldsymbol{\theta}}$, stepsizes $\alpha_t,\beta_t>0$, initial points $\boldsymbol{x}=0$, $\mathbf{z}=0$. A constant $\ell_F$.
		\STATE \textbf{policy parameter $\boldsymbol{\theta}\in\mathbb{R}^d$} 
		\STATE  Generate episodes $\zeta_i=\{(s_k, a_k)\}$ from $i=1,\cdots, n$ following $\pi_\theta(a|s)$
		\STATE  For $t = 0,1,2,...$ until some stopping criterion is met: 
		\STATE \hspace{0.5 cm} \textbf{Sample} $(s_k,a_k)$ from the data set
		\STATE \hspace{0.5	cm} \textbf{Update} \vspace{-0.35cm}
		\begin{align}
		& \mathbf{z}^{t+1} \leftarrow \Proj_{\ell_F}\left\{\mathbf{z}^t - \frac{\alpha_t}{1-\gamma} \mathbf{1}_{s_k,a_k} + \alpha_t \nabla F^*(\mathbf{z}^t)  \right\}
		\label{alg-1}\\
		&\boldsymbol{x}^{t+1}\leftarrow  \boldsymbol{x}^t + \beta_t\left[ \sum_{a\in\cA}Q^{\pi_\theta}(s_k,a; \mathbf{z}^t)\cdot\nabla_\theta\pi_\theta(a|s_k) - \boldsymbol{x}^{t}\right]
		\label{alg-2}
		\end{align} 	
		\STATE \textbf{Output:} the last iterate $\boldsymbol{x}$
	\end{algorithmic}
\end{algorithm}
It is worth noting that 
we have we omit the term $\delta\nabla\tilde V(\theta;z)^\top x$
when computing the gradient w.r.t. $z$ in \eqref{alg-1}. Note that for the iterates $x^t$ are all well bounded, then $\delta\nabla\tilde V(\theta;z^t)^\top \!x^t\! =\! \cO(\delta)$, which is negligible when $\delta\to0$.

\subsection{Special cases of policy gradient computation}
\label{appdx:special_cases}
We give several examples of the policy gradient for special cases of the general utility in \eqref{prob:dual-para-pi}.


\paragraph{Linear utility} 
The simplest, where $F(\lambda) = \langle \lambda, r\rangle$ [cf. \eqref{eq:value}], we have $F^*(z) = 0$ if $z= c \cdot r$ for some scalar $c$ and $F^*(z)=\infty$ otherwise. In this case $z^*=r$ and Theorem \ref{thm-PG} recovers the known policy gradient theorem for the risk-neutral MDP \eqref{eq:value}, that is
$\nabla_{\theta} R(\pi_{\theta})= { \nabla_{\theta} V(\theta;r)} $.

\paragraph{Constrained MDPs}
By contrast, in Example \ref{eg:cmdp}, i.e., when a constraint $\mathbb{E}^{\pi}\left[\sum^{\infty}_{t=0} \gamma^t c(s_t,a_t) \right] \leq C$ on the accumulation of costs $c(s_t,a_t)$ is present, and we may enforce it approximately with a $\log$ barrier by defining 
%
\begin{equation}\label{eq:logarithmic_barrier}
R(\pi_\theta) = \langle r,\lambda(\theta)\rangle + \beta \log\left(C-\langle c,\lambda(\theta)\rangle\right) = V(\theta;r) + \beta \log\left(C-V(\theta;c)\right),
\end{equation}
where $\beta$ is a regularization parameter,
in which case the policy gradient takes the form
$$\nabla R(\pi_\theta) = \nabla_{\theta}V(\theta;r) - \beta \frac{\nabla_{\theta}V(\theta;c)}{C-V(\theta;c)}.$$
Estimating the policy gradient $R$ of constrained MDP consists of estimating two policy gradients $\nabla_{\theta}V(\theta;c)$ and $\nabla_\theta V(\theta;r)$ and accumulated reward $V(\theta;c)$. 

\paragraph{Minimum eigenvalue}
For case \eqref{prob:exploration-2}, define $\Phi(\lambda^{\pi_{\theta}}) = \sum_{s,a} \lambda_{sa}^{\pi_{\theta}} \cdot\phi(s,a)\phi(s,a)^{\top}$. Then  $\Phi(\lambda^{\pi_{\theta}})$ is symmetric and positive semidefinite, since $\lambda^{\pi_{\theta}}\geq0$. By using Rayleigh principle, we have 
\begin{eqnarray}
R(\pi_{\theta}) = \sigma_{min}\left(\Phi(\lambda^{\pi_{\theta}}) \right) = \min_{\|u\| = 1} u^\top \Phi(\lambda^{\pi_{\theta}}) u = \min_{\|u\| = 1}\sum_{s,a}\lambda_{sa}^{\pi_{\theta}}|\phi(s,a)^{\top} u |^2.
\end{eqnarray}
which is the minimum of a family of linear function in $\lambda$. Let $v^{(1)},...,v^{(k)}$ be a group of orthonormal bases of the eigenspace of $\Phi(\lambda^{\pi_{\theta}})$ corresponding to the minimum eigenvalue. Then define $k$ vectors as $r^{(i)}(s,a) = |\phi(s,a)^{\top} v^{(i)} |^2, \forall s,a$, $i = 1,...,k$. Then the Fr\'echet superdifferential of $R$ at $\theta$ is 
$$\hat\partial_{\theta} R(\pi_{\theta}) = \left\{\nabla_{\theta} V(\theta; r): r\in\mathrm{conv}(r^{(1)},...,r^{(k)})\right\},$$
where $\mathrm{conv}(\cdot)$ denotes the convex hull of a group of vectors. When the multiplicity of the minimum eigenvalue is 1, then $R(\cdot)$ is differentiable at this point and $\hat{\partial}_\theta R(\cdot) = \{\nabla_\theta R(\cdot)\}$.

\paragraph{Entropy maximization}
For the entropy \eqref{prob:exploration-1}, its Fenchel dual takes the form
$$F^*(z) 
= -\sum_{sa}\exp\big\{-\frac{z_{sa}}{1-\gamma}-1\big\}.$$ 
%

\paragraph{Learning to mimic a distribution}  For the KL divergence to a prior $\mu$ in \eqref{risk_not_r}, we have
$$F^*(z) = \begin{cases}
-\sum_{s}\mu_s\exp\big\{-\frac{z_{s1}}{1-\gamma}-1\big\} & \mbox{if} \quad z_{sa_1} = z_{sa_2}\,\,\,\,\forall s\in\cS, a_1,a_2\in\cA,\\
-\infty & \mbox{otherwise.}
\end{cases}$$


\section{Additional Details of Experiments}\label{apx_experiments}
\subsection{Details of Environment}
	OpenAI Frozen Lake is a finite-state action problem. The standard state consists of $\{S,F,H,G\}$, to which we add an additional state $C$ which is visualized in Fig. \ref{fig:cmdp_occupancy}.  At each step, an agent selects an action $a\in\cA$, which consists of one of four directions (up, down, left, right), which may be enumerated as $\{1,\dots,4\}$. The reward is null at all Frozen $F$ spaces, the start $S$ location, and the Holes $H$ in the lake. If the agent enters a hole, the episode terminates, and hence null reward is accumulated for this trajectory. The only positive reward is $1$ and may be obtained when reaching the goal state $G$. Our augmentation is that costly states $C$ have been added, which incur reward $-0.4$ to represent, for instance, obstacles. We note that only for the cumulative return and its constrained variants, or other utilities that are defined in terms of the problem's inherent reward do these quantities matter. That is, for the entropy maximization problem, there is no reward associated with any state in the usual sense. The MDP transition model is unknown and defined by the OpenAI environment, a simulation oracle that provides state-action-reward triples. 

Throughout all experiments, for simplicity, we considered a softmax policy parameterization. For this parameterization, the policy takes the form $\pi_{\theta}(s\mid a) = e^{\theta_{sa}}/(\sum_{a'} e^{\theta_{sa'}})$ for $\theta\in\RR^{|\mathcal{S}|\times|\mathcal{A}|}$. For the Frozen lake environment in this paper, we have $|\mathcal{S}|=16$ and $|\mathcal{A}|=4$. 

\subsection{Computing the True Policy Gradient}\label{apx_experiments_benchmark}
 For comparison, we compute the true policy gradient by using a baseline approach based on the chain rule and a variant of REINFORCE \cite{sutton2000policy}: the second factor on the right-hand side of \eqref{chain-rule} is exactly computed using REINFORCE $\nabla_\theta\lambda_{sa}(\theta)$, whereas the first, $\frac{\partial \F(\lambda(\theta))}{\partial \lambda_{sa}}$, is computed using an additional Monte Carlo rollout. We denote as $x^*$ the result of this procedure and use it as ground truth. In Figure \ref{fig:convergence} we display the evolution of its norm difference $\|\hat{x}_n^\star-\hat{x}_{n-1}^\star\|$ as the sample size $n$ increases. That it approaches null with the sample size implies that this brute force Monte Carlo variant of REINFORCE is convergent, and hence is a reasonable benchmark comparator.

\begin{figure}[h]
\vspace{-4mm}
\centering
	\subfigure[Convergence of $x^\star$]{\includegraphics[scale=0.4]{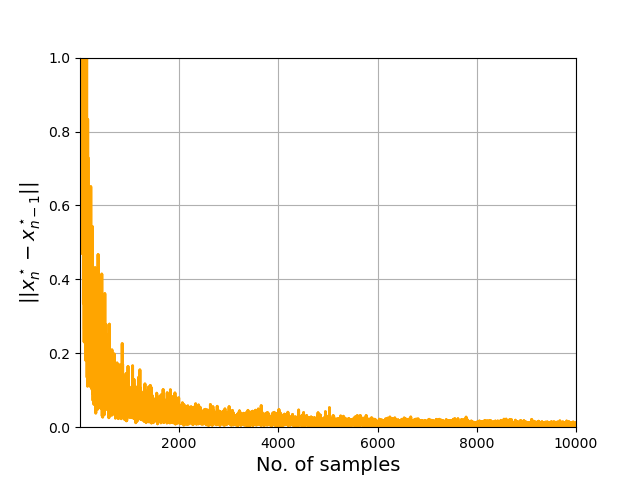}\label{fig:convergence}}
	%
		\caption{	Fig. \ref{fig:convergence} displays the convergence of a generalization of REINFORCE-based gradient estimator for \eqref{chain-rule} in terms of its difference $\|\hat{x}_n^\star-\hat{x}_{n-1}^\star\|$ as the number of processed trajectories $n$ increases, which converges to null, certifying $\hat{{x}}_n^\star$ as a baseline.}
\end{figure}

\subsection{Details about Maximum Entropy Exploration}\label{subsec:entropy}
For this problem instance, i.e., \eqref{prob:exploration-1} from Example \ref{eg:explore}, we also consider the state space defined by Frozen Lake, but note that the reward as defined by the environment is now a moot point. This is because each state contributes positive entropy, with the exception of the holes in the lake, which terminate the episode. We visualize this setup at the bottom layer of Fig. \ref{fig:entropy_occupancy}. The lower middle layer visualizes the occupancy measure associated with a uniform policy. Moreover, the upper middle layer visualizes the ``pseudo-reward" $z$ for each point in the state space. This quantity is computed in terms of the Fenchel dual of the entropy -- see Appendix \ref{appdx:special_cases}, and the occupancy measure associated with the output of Algorithm \ref{alg:MCDPG} at the end of training is visualized at the top layer. To obtain this result, we run it for $10^5$ total episodes, and for each episode we evaluate the entropy using \eqref{prob:exploration-1}. We consider a constant step-size $\alpha=0.01, \beta= 0.1$, and $\eta=0.001$ throughout this experiment.

\subsection{Details about the Constrained Markov Decision Process}\label{subsec:CMDP}

In this subsection, we elaborate upon the implementation of Example \ref{eg:cmdp}, specifically, \eqref{prob:dual-cmdp} and its approximation using a logarithmic barrier as detailed in \eqref{eq:logarithmic_barrier}. We consider the problem of navigating through the FrozenLake environment as shown in Fig. \ref{fig:cmdp_occupancy}(bottom): we seek to reach the goal state $G$ (reward $=1$) from the starting location $S$ (reward $=0$), navigating along $F$ frozen spaces (reward $=0$), while avoiding  locations marked $C$ (reward $=-0.2$) that denote costly states (obstacles) and $H$ holes. 

We consider two approaches to the problem: first, we focus on optimizing the standard expected cumulative return \eqref{eq:value}, whose associated state occupancy measure is given in Fig. \ref{fig:cmdp_occupancy}(middle); second, we consider imposing constraints to avoid costly states [cf. \eqref{prob:dual-cmdp}] via a logarithmic barrier \eqref{eq:logarithmic_barrier}, whose resulting occupancy measure is depicted in Fig. \ref{fig:cmdp_occupancy}(top). Bluer/yellower colors denote higher/lower likelihoods, respectively. We observe that imposing constraints yields policies whose probability mass is concentrated away from constraints and instead along paths from the start to the goal. Thus, Algorithm \ref{alg:MCDPG} combined with a policy search scheme \eqref{defn:grad-proj} may be used to solve CMDPs.

This trend is corroborated in Fig. \ref{fig:cmdp_convergence}, which depicts the reward \ref{fig:CMDP_value} and cost \ref{fig:CMDP_cost} accumulation during test trajectories as a function of training index for Algorithm \ref{alg:MCDPG} for the cumulative return \eqref{eq:value} as compared with a logarithmic barrier imposed to solve CMDP \eqref{prob:dual-cmdp} for different penalty parameters $\beta$. We may observe that without imposing any constraint ($\beta=0$ in red), one achieves the highest reward, but incurs the most costs, i.e., hits obstacles most often, a form of ``reckless boldness." Larger $\beta$ means higher penalty for the constraints, and hence $\beta=4$ incurs lower cost and lower reward. We further added the curves for $\beta=1$ and $\beta=2$ for comparison. 

For all results reported in Fig. \ref{fig:cmdp_convergence}, we run the algorithm for $10K$ total training steps in the form of episodes. For each episode, we run a number of evaluation (test) trajectories in order to determine their merit, both in terms of reward and cost accumulation. Put more simply, we evaluate the performance averaged over a few test trajectories as a function of episode number and report its average over last $20$ episodes to show the trend. This is to illuminate policy improvement in its various forms (reward/cost accumulation) during training. Moreover, the algorithm is run with constant step-size $\eta=0.1$ throughout this experiment.

\section{Proof of Theorem \ref{thm-PG}}
\begin{proof}
   First note that for any $z\in \mathbb{R}^{SA}, x\in \mathbb{R}^{d}$, we have    %
    \begin{align}
	  V(\theta;z) & = \langle z, \lambda(\theta)\rangle\, , \nonumber \\
	 \nabla_{\theta} V(\theta;z)^{\top} x & = \langle z, \nabla_{\theta} \lambda(\theta)x \rangle\,,
	 \label{eq:appg}
	\end{align}
	where $\nabla_{\theta} \lambda(\theta)$ is the $SA\times d$ Jacobian matrix, the first identity holds by definition, and the second holds by directly differetiating the first indentity and product it with $x$.

Consider the saddle point problem in \eqref{thm-PG-2} for fixed  $0<\delta<1$. 
Let $G$ be any constant such that  $\|\nabla F(\lambda(\theta))\|_{\infty} < G$.
	Define 
	\begin{eqnarray}
	\label{thm-PG-proof-1}
	(x^*(\delta),z^*(\delta))& :=  \argmax_{x} \argmin_{\|z\|_\infty\leq G} \big\{ V(\theta;z) + \delta\nabla_{\theta} V(\theta;z)^{\top} x- F^*(z) -  \frac{\delta}{2}\|x\|^2 \big\}.
	\end{eqnarray}
	
	Note in \eqref{thm-PG-proof-1} we added the auxiliary  constraint set $\{z: \|z\|_\infty\leq G\}$, and later we will show that this constraint is inactive for all $\delta$ sufficiently small. We will also show that $(x^*(\delta),z^*(\delta))$ are bounded for all $\delta$ sufficiently small.
	
	By the first-order stationarity condition, we have 
	$$x^*(\delta) = \nabla_{\theta}V(\theta;z^*(\delta)).$$
	Note that $\nabla_{\theta}V(\theta;\cdot)$ is a linear function of $z$, thus there exists  $B>0$ such that $\|\nabla_{\theta}V(\theta;z)\|\leq B$ for all $ z\in\{\|z\|_\infty\leq G\}$.  And consequently $\|x^*(\delta)\|\leq B$ for all $\delta>0$. 
	
	For all $x\in\{\|x\|\leq 2B\}$, we have
    $$\lim_{\delta\to 0_+} \lambda(\theta) + \delta\nabla_{\theta} \lambda(\theta) x = \lambda(\theta) .$$
	Therefore, there exists some small $\delta_0>0$, such that for all $\delta<\delta_0$, the vector $\lambda(\theta) + \delta\nabla_{\theta} \lambda(\theta)x$ belongs to the neighborhood on which $F$ is differentiable and  
	$$\|\nabla F(\lambda(\theta) + \delta\nabla_{\theta} \lambda(\theta)x)\|_\infty<G,\quad\forall\ \ x\in\{x:\|x\|\leq 2B\}.$$ 
	In this case, we consider the unconstrained solution, for $\|x\|\leq 2B$, defined by 
	$$z^*(x;\delta) := \argmin_{z} V(\theta;z) + \delta\nabla_{\theta} V(\theta;z)^{\top} x- F^*(z) = \nabla F\big(\lambda(\theta) + \delta\nabla_{\theta} \lambda(\theta)x\big),$$
	and observe that the unconstrained solution satisfies $\|z^*(x;\delta)\|_{\infty}<G$, and consequently the constraint $\|z\|_{\infty}\leq G$ is not active. 
	Therefore, for $\delta<\delta_0$, we can equivalently rewrite \eqref{thm-PG-proof-1} as 
	\begin{eqnarray}
	\label{thm-PG-proof-1'}
	x^*(\delta)& := &  \argmax_{\|x\|\leq2B} \min_{z} \left\{ V(\theta;z) + \delta\nabla_{\theta} V(\theta;z)^{\top} x- F^*(z) -  \frac{\delta}{2}\|x\|^2 \right\}\\ 
	& = &  \argmax_{\|x\|\leq2B} F\left(\lambda(\theta) + \delta\nabla_{\theta} \lambda(\theta) x\right) -  \frac{\delta}{2}\|x\|^2\,, \nonumber
	\end{eqnarray}
Recall that we showed $\|x^*(\delta)\|\leq B$, therefore the constraint  $\|x\|\leq2B$ is also inactive and removable. 	
Therefore $x^*(\delta)$ is equivalent to the unconstrained min-max solution, for all $\delta$ sufficiently small, and 
Fenchel duality together with the first-order stationarity condition implies
   \begin{align*}
{x^*(\delta)} &=  \argmax_{x} \inf_z \left\{ V(\theta;z) + \delta\nabla_{\theta} V(\theta;z)^{\top} x- F^*(z) -  \frac{\delta}{2}\|x\|^2 \right\}\nonumber \\
&=\nabla_{\theta}\lambda(\theta)^\top\nabla F\big(\lambda(\theta) + \delta\nabla_{\theta}\lambda(\theta)x^*(\delta)\big).
   \end{align*}
   By using the fact that $\nabla F$ is continuous at $\lambda(\theta)$ and $x^*(\delta)$ is bounded, by letting $\delta\to 0$ on both sides, we get
   \begin{align*}
	\lim_{\delta\to 0_+}  {x^*(\delta)} 
	&= 
	\lim_{\delta\to 0_+} 
	\nabla_{\theta}\lambda(\theta)^\top\nabla F\big(\lambda(\theta) + \delta\nabla_{\theta}\lambda(\theta)x^*(\delta)\big)\\
	&= 
	\nabla_{\theta}\lambda(\theta)^\top\nabla F\big(\lambda(\theta) \big) \\
	& = \nabla R(\theta)\,,
   \end{align*}
   where the last equality uses the chain rule.

\end{proof}

\section{Proof of Theorem \ref{thm-PGestimate}}
\begin{proof}
	First, let us denote the expression in \eqref{thm-PG-2} for fixed $0<\delta<1$ as
	\begin{equation}
	\label{prob:SSP}
	(x^*(\delta),z^*(\delta)) = \argmax_{x}\argmin_{\|z\|_\infty\leq \ell_F} V(\theta;z) + \delta\nabla_\theta V(\theta;z)^\top x - F^*(z) - \frac{\delta}{2}\|x\|^2,
	\end{equation}
	and its approximation with empirically estimated value functions and their gradients in \eqref{eq:empirical_value_estimate} as
	\begin{equation}
	\label{prob:ESP}
	(\hat x(\delta),\hat z(\delta)) = \argmax_{x}\argmin_{\|z\|_\infty\leq\ell_F} \tilde V(\theta;z) + \delta\nabla_\theta \tilde V(\theta;z)^\top x - F^*(z)- \frac{\delta}{2}\|x\|^2.
	\end{equation}
	%
Then we decompose the entity $\EE\left[\left\|\hat \nabla_{\theta} R(\pi_\theta)- \nabla_{\theta} R(\pi_\theta)\right\|^2\right] $ into three terms by adding and subtracting (i) $ x^*(\delta) $ and (ii)  $ \hat{x}(\delta) $, which we then establish depends on the difference between (iii) $\hat{z}(\delta)$ and $z^*(\delta)$. Taken together with computing the limit of the right-hand side as  $\delta\rightarrow 0$ we obtain the result. Each of these steps is analyzed independently, whose estimation errors are derived in the following lemma.
%
\begin{lemma}\label{lemma:theorem2}
Consider $(x^*(\delta),z^*(\delta))$ and $(\hat x(\delta),\hat z(\delta)) $ as defined in \eqref{prob:SSP}-\eqref{prob:ESP}, respectively. Under the technical conditions stated in Theorem \ref{thm-PGestimate}, their respective estimation errors satisfy:
\begin{enumerate}[label=(\roman*)]

\item	$ \Big\| x^*(\delta) - \nabla_\theta R(\pi_\theta)\Big\|^2 = \cO(\delta^2).$ \label{lemma:step1}
\item	$\EE\left[\Big\| x^*(\delta) -  \hat x(\delta)\Big\|^2\right] \leq \frac{2C^2\|z^*(\delta)\|_\infty^2}{(1-\gamma)^4}\cdot\left(\frac{\gamma^{2K}}{(1-\gamma)^2} + \frac{1}{n}\right) + \frac{2C^2}{(1-\gamma)^4}\cdot\EE\left[\left\|z^*(\delta)-\hat z(\delta)\right\|_\infty^2\right].$ \label{lemma:step2}
\item	$\EE\left[\|\hat z(\delta)-z^*(\delta)\|^2_\infty\right]\leq\cO\left(\frac{L_F^2}{n(1-\gamma)^2} + \frac{L_F^2\ell_{F^*}^2}{n} + \frac{L_F^2\delta^2 + L_F\delta}{n}\right).$ \label{lemma:step3}
\end{enumerate}
\end{lemma}
Combining the three steps and the fact that $\|z^*(\delta)\|_\infty\leq\ell_F$ yields
$$\EE\left[\left\| \hat x(\delta) - \nabla_{\theta} R(\theta)\right\|^2\right] \leq \cO\left(\frac{C^2(\ell_F^2 + L_F^2\ell_{F^*}^2)}{n(1-\gamma)^4} + \frac{C^2L_F^2}{n(1-\gamma)^6}\right) + \cO(\delta^2 + \delta/n + \gamma^K).$$
Let $\delta\to0$, we get 
$$\EE\left[\left\|\hat \nabla_{\theta} R(\pi_\theta)- \nabla_{\theta} R(\pi_\theta)\right\|^2\right] \leq \cO\left(\frac{C^2(\ell_F^2 + L_F^2\ell_{F^*}^2)}{n(1-\gamma)^4} + \frac{C^2L_F^2}{n(1-\gamma)^6}\right) + \cO(\gamma^K).$$

Lemma \ref{lemma:theorem2}\ref{lemma:step1} - \ref{lemma:step3} is proved in the next subsection. For the ease of notation, we will simply denote $x^*$ and $\hat x$ instead of $x^*(\delta)$ and $\hat x(\delta)$. Similarly, we denote $z^*$ and $\hat z$ instead of $z^*(\delta)$ and $\hat z(\delta)$. 

\end{proof}
\subsection{Preliminary Technicalities}\label{subsec:properties}

\textbf{Linearity property}.
The functions $Q$, $V$ and $\nabla_{\theta}V$ are linear in the reward function. Namely, for any $\alpha,\alpha'\in \RR$ and $r,r'\in\RR^{|\cS||\cA|}$, 
$$\alpha \nabla_{\theta}V(\theta;r) + \alpha' \nabla_{\theta}V(\theta;r') = \nabla_{\theta}V(\theta;\alpha r + \alpha' r').$$
Similar identities holds for $Q^{\pi_{\theta}}(s,a;\cdot)$ and $V(\theta;\cdot)$. For the stochastic estimators $\nabla_{\theta} \tilde V(\theta;r;\zeta)$, it is straightforward to check that the linearity property is still true. \\
\textbf{Upperbounding $Q$ and $V$}. Given an arbitrary reward function $r$, the upper bounds of $Q$ and $V$ functions are 
$$|Q^{\pi_\theta}(s,a;r)|\leq \frac{\|r\|_\infty}{1-\gamma}\qquad\mbox{and}\qquad |V(\theta;r)|\leq\frac{\|r\|_\infty}{1-\gamma}.$$ \\
\textbf{Uniform upperbounds for estimators.}
Given any sample path $\zeta = \{(s_k,a_k)\}_{k=0}^K$, the estimators $\tilde V(\theta;z;\zeta)$ and $\nabla_\theta\tilde{V}(\theta;z;\zeta)$ are upper bounded by 
\begin{equation}
\label{thm2-uperbound-sto}
\tilde V(\theta;z;\zeta)|\leq \frac{\|z\|_\infty}{1-\gamma}\qquad\mbox{and}\qquad \|\nabla_{\theta}\tilde V(\theta;z;\zeta)\|\leq \frac{C\|z\|_\infty}{(1-\gamma)^2}.
\end{equation}
Consequently, as the sample averages of $\tilde V(\theta;z;\zeta_i)$ and $\nabla_\theta\tilde{V}(\theta;z;\zeta_i)$, we also have 
\begin{equation}
\label{thm2-uperbound-sto'}
|\tilde V(\theta;z)|\leq \frac{\|z\|_\infty}{1-\gamma}\qquad\mbox{and}\qquad \|\nabla_{\theta}\tilde V(\theta;z)\|\leq \frac{C\|z\|_\infty}{(1-\gamma)^2}
\end{equation}
for any set of sample paths $\{\zeta_i\}_{i=1}^n$.
\begin{proof}
For $\tilde V(\theta;z;\zeta)$,  for any $z$,
\begin{eqnarray*}
	|\tilde V(\theta;z;\zeta)| = \left|\sum^K_{k=0} \gamma^k\cdot z(s_k,a_k)\right| \leq \sum^K_{k=0} \gamma^k\|z\|_\infty \leq \frac{\|z\|_\infty}{1-\gamma}
\end{eqnarray*}
For $\nabla_{\theta}\tilde V(\theta;z;\zeta)$,  for any $z$,
\begin{eqnarray*}
\|\nabla_{\theta}\tilde V(\theta;z;\zeta)\| & = & \left\|\sum^K_{k=1}\sum_{a\in\cA} \gamma^k\cdot Q(s_k,a; z)\nabla_{\theta}\pi_{\theta}(a|s_k)\right\|\\
& \leq & \sum^K_{k=1}\gamma^k\cdot\left\|\sum_{a\in\cA}  Q(s_k,a; z)\nabla_{\theta}\pi_{\theta}(a|s_k)\right\|\\
& \leq & \sum^K_{k=1}\gamma^k\cdot\max_{\|u\|_\infty\leq\frac{\|z\|_\infty}{1-\gamma}}\left\|\pi_{\theta}(\cdot|s_k)u\right\|\\
& \leq & \frac{C\|z\|_\infty}{(1-\gamma)^2}.
\end{eqnarray*}
\end{proof}
\subsection{Proof of Lemma \ref{lemma:theorem2}\ref{lemma:step1}.}
Consider the problem \eqref{prob:SSP}. First let us ignore the requirement that $\|z\|_\infty \leq \ell_F$. For this series of unconstrained problem, Theorem \ref{thm-PG} suggests that 
$$\lim_{\delta\to 0_+} x^*(\delta) = \nabla_\theta R(\pi_{\theta}).$$
Consequently, $\lim_{\delta\to 0_+} \lambda(\theta) + \delta\nabla_{\theta}\lambda(\theta)x^*(\delta) = \lambda(\theta)$. Because $\|\lambda(\theta)\|_1 = (1-\gamma)^{-1}$, $\exists \delta_0>0$ s.t. when $\delta<\delta_0$ we have 
$$\|\lambda(\theta) + \delta\nabla_{\theta}\lambda(\theta)x^*(\delta)\|_1\leq \frac{2}{1-\gamma}.$$
According to condition (i) of this theorem, we have
$$\|\nabla F(\lambda(\theta) + \delta\nabla_{\theta}\lambda(\theta)x^*(\delta))\|_\infty\leq \ell_F.$$
It is worth noting that $z^*(\delta)=\nabla F(\lambda(\theta) + \delta\nabla_{\theta}\lambda(\theta)x^*(\delta))$ is also the solution to the unconstrained version of \eqref{prob:SSP}. Therefore we have $\|z\|_\infty\leq\ell_F$, so that we can add this to the constraint without changing the optimal solutions. By the intermediate result in the proof of Theorem \ref{thm-PG}, we have
\[
x^*(\delta)=  \nabla_{\theta}\lambda(\theta)^\top\nabla F\big(\lambda(\theta) + \delta\nabla_{\theta}\lambda(\theta)x^*(\delta)\big).
\]
Consequently, 
by the Lipschitz continuity of $\nabla F$, we have 
\begin{eqnarray*}
  \Big\|x^*(\delta) - \nabla_\theta R(\theta)\Big\|^2 & = & \Big\|\nabla_{\theta}\lambda(\theta)^\top\nabla F\big(\lambda(\theta) + \delta\nabla_{\theta}\lambda(\theta)x^*(\delta)\big)-\nabla_{\theta}\lambda(\theta)^\top\nabla F\big(\lambda(\theta) \big)\Big\|^2\nonumber\\
  & \leq & \|\nabla_{\theta}\lambda(\theta)^\top\|_{\infty,2}\cdot\Big\|\nabla F\big(\lambda(\theta) + \delta\nabla_{\theta}\lambda(\theta)x^*(\delta)\big)-\nabla F\big(\lambda(\theta) \big)\Big\|^2_\infty\nonumber\\
  & \leq & L_F\|\nabla_{\theta}\lambda(\theta)^\top\|^2_{\infty,2}\cdot\Big\|\delta\nabla_{\theta}\lambda(\theta)x^*(\delta)\Big\|^2_1\nonumber\\
  &  = & \cO(\delta^2).
\end{eqnarray*}
as stated in Lemma \ref{lemma:theorem2}\ref{lemma:step1}.
In the last step, we used the fact that $x^*(\delta)$ is bounded because $x^*(\delta)\to\nabla_\theta R(\pi_{\theta})$.\hfill $\qed$

\subsection{Proof of Lemma \ref{lemma:theorem2}\ref{lemma:step2}.}
By the first order stationarity condition of the problems \eqref{prob:SSP}-\eqref{prob:ESP}, we know 
%
$$ x^*= \nabla_{\theta}V(\theta;z^*)\qquad\mbox{and}\qquad \hat x = \nabla_{\theta}\tilde V(\theta;\hat z).$$
Consider the norm-difference between the preceding quantities:
\begin{eqnarray}
\label{thm-pgestimate-1}
\EE\left[\big\| x^* -  \hat x\big\|^2\right] \leq 2\EE\left[\|\nabla_\theta V(\theta;z^*) - \nabla_{\theta} \tilde V(\theta;z^*)\|^2\right] +  2\EE\left[\|\nabla_{\theta} \tilde V(\theta;z^*) - \nabla_{\theta} \tilde V(\theta;\hat z)\|^2\right].
\end{eqnarray}
To bound the term $\EE\left[\|\nabla_\theta V(\theta;z^*) - \nabla_{\theta} \tilde V(\theta;z^*)\|^2\right]$, recall the definition \eqref{eq:empirical_value_estimate}:
$$\nabla \tilde V(\theta;z) := \frac{1}{n}\sum_{i=1}\nabla_{\theta}V(\theta;z;\zeta_i)= \frac{1}{n}\sum_{i=1}^n\sum^K_{k=1}\sum_{a\in\cA} \gamma^kQ(s_k^{(i)},a; z)\nabla_{\theta}\pi_{\theta}(a |s_k^{(i)}).$$
Consider the first term on the right-hand side of \eqref{thm-pgestimate-1}. Add and subtract $ \EE\left[\nabla_{\theta} \tilde V(\theta;z^*)\right]$ and use the fact that $ \EE\left[\nabla_{\theta} \tilde V(\theta;z^*)\right]= \nabla_{\theta}  V(\theta;z^*)$, i.e., the bias-variance decomposition identity, to write
\begin{eqnarray}
\label{thm-pgestimate-2}
&  \EE \!\!&\!\!\left[\left\|\nabla_\theta V(\theta;z^*) - \nabla_{\theta} \tilde V(\theta;z^*)\right\|^2\right] \\
& = & \left\|\nabla_\theta V(\theta;z^*) - \EE\left[\nabla_{\theta} \tilde V(\theta;z^*)\right]\right\|^2 + \EE\left[\left\|\nabla_{\theta} \tilde V(\theta;z^*)-\EE\left[\nabla_{\theta} \tilde V(\theta;z^*)\right]\right\|^2\right]\nonumber.
\end{eqnarray}
For the first (squared bias) term on the right-hand side of \eqref{thm-pgestimate-2}, denote $d^\pi_{\xi,K}(s) = (1-\gamma)\sum^K_{t=0} \gamma^t\mathbf{Prob}(s_t = s|\pi,s_0\sim \xi)$. Then it is straightforward that 
$\sum_{s}|d^\pi_{\xi,K}(s)-d^\pi_{\xi}(s)|\leq \frac{\gamma^K}{1-\gamma}$.  As a result, we know 
\begin{eqnarray}
\label{thm-pgestimate-3}
& &\hspace{-2.5cm} \left\|\nabla_\theta V(\theta;z^*) - \EE\left[\nabla_{\theta} \tilde V(\theta;z^*)\right]\right\|^2 \\
& = &\frac{1}{(1-\gamma)^2} \left\|\sum_{s}\left(d^{\pi}_\xi(s) - d^{\pi}_{\xi,K}(s)\right)\sum_a Q^{\pi_{\theta}}(s,a;z^*)\nabla_{\theta}\pi_{\theta}(a|s)\right\|^2\nonumber\\
& = & \frac{1}{(1-\gamma)^2} \left(\sum_{s}|d^{\pi}_\xi(s) - d^{\pi}_{\xi,K}(s)|\cdot\Big\|\sum_a Q^{\pi_{\theta}}(s,a;z^*)\nabla_{\theta}\pi_{\theta}(a|s)\Big\|\right)^2\nonumber\\
& = & \frac{1}{(1-\gamma)^2} \left(\sum_{s}|d^{\pi}_\xi(s) - d^{\pi}_{\xi,K}(s)|\cdot\Big\|\sum_a Q^{\pi_{\theta}}(s,a;z^*)\nabla_{\theta}\pi_{\theta}(a|s)\Big\|\right)^2\nonumber\\
& \leq & \frac{1}{(1-\gamma)^2} \left(\sum_{s}|d^{\pi}_\xi(s) - d^{\pi}_{\xi,K}(s)|\cdot\max_{\|u\|_\infty\leq\frac{\|z^*\|_\infty}{1-\gamma}}\|\nabla_{\theta}\pi(\cdot|s) u\|\right)^2\nonumber\\
& \leq & \frac{\|\nabla_{\theta}\pi(\cdot|s)\|^2_{\infty,2}\cdot\|z^*\|_\infty^2}{(1-\gamma)^4} \left(\sum_{s}|d^{\pi}_\xi(s) - d^{\pi}_{\xi,K}(s)|\right)^2\nonumber\\
& \leq & \frac{C^2\|z^*\|_\infty^2}{(1-\gamma)^6}\gamma^{2K}.\nonumber
\end{eqnarray}
%
Next, we consider the second (variance) term on the right-hand side of \eqref{thm-pgestimate-2}. By substituting \eqref{eq:empirical_value_estimate} in for $\nabla_{\theta} \tilde V(\theta;z^*)$ to rewrite it in terms of trajectories $\zeta_i$, we have
 \begin{eqnarray}
\label{thm-pgestimate-4}
\EE\left[\left\|\nabla_{\theta} \tilde V(\theta;z^*)-\EE\left[\nabla_{\theta} \tilde V(\theta;z^*)\right]\right\|^2\right] 
& = & \frac{1}{n}\EE\left[\left\|\nabla_{\theta} \tilde V(\theta;z^*;\zeta_i)-\EE\left[\nabla_{\theta} \tilde V(\theta;z^*;\zeta_i)\right]\right\|^2\right]\nonumber\\
&\leq & \frac{1}{n}\EE\left[\left\|\nabla_{\theta} \tilde V(\theta;z^*;\zeta_i)\right\|^2\right]\nonumber\\
&\leq&  \frac{C^2\|z^*\|_\infty^2}{n(1-\gamma)^4}\nonumber.
\end{eqnarray}
The first inequality comes from crudely upper-bounding the bias by the estimator itself. The last equality uses \eqref{thm2-uperbound-sto}.
%

Now, returning focus to the second term in the bound \eqref{thm-pgestimate-1}, by the linearity of the stochastic estimators with respect to the differential and \eqref{thm2-uperbound-sto'}, we have 
\begin{eqnarray}
\left\|\nabla_{\theta} \tilde V(\theta;z^*) - \nabla_{\theta} \tilde V(\theta;\hat z)\right\|^2 = \left\|\nabla_{\theta} \tilde V(\theta;z^*-\hat z) \right\|^2\leq \frac{C^2\|z^*-\hat z\|^2_\infty}{(1-\gamma)^4}.\nonumber
\end{eqnarray}

%
Taking the expectation after squaring both sides yields
\begin{eqnarray}
\label{thm-pgestimate-5}
\EE\left[\left\|\nabla_{\theta} \tilde V(\theta;z^*) - \nabla_{\theta} \tilde V(\theta;\hat z)\right\|^2\right] \leq \frac{C^2}{(1-\gamma)^4}\EE\left[\|z^*-\hat z\|_\infty^2\right].
\end{eqnarray}

Combining inequalities \eqref{thm-pgestimate-1}, \eqref{thm-pgestimate-2}, \eqref{thm-pgestimate-3}, \eqref{thm-pgestimate-4}, \eqref{thm-pgestimate-5} yields 
$$\EE\left[\big\| x^* -  \hat x\big\|^2\right] \leq \frac{2C^2\|z^*\|_\infty^2}{(1-\gamma)^4}\cdot\left(\frac{\gamma^{2K}}{(1-\gamma)^2} + \frac{1}{n}\right) + \frac{2C^2}{(1-\gamma)^4}\cdot\EE\left[\left\|z^*-\hat z\right\|_\infty^2\right].$$
which is as stated in Lemma \ref{lemma:theorem2}\ref{lemma:step2}. \hfill $\qed$

\subsection{Proof of Lemma \ref{lemma:theorem2}\ref{lemma:step3}.}

In this section we will apply the generalization bound for stochastic saddle points from \cite{Gen_SPP} to bound the term $\EE[\|\hat z - z^*\|_\infty^2]$. To achieve this, we need a compact feasible region for $x$. Note that for problems \eqref{prob:SSP} and \eqref{prob:ESP}, the solutions $x^*$ and $\hat{x}$ has the form 
$$x^* = \nabla_{\theta}{V}(\theta; z^*)\qquad\mbox{and}\qquad\hat x = \nabla_{\theta}\tilde{V}(\theta;\hat z).$$
Due to \eqref{thm2-uperbound-sto'} and the constraint that $\|z\|_\infty\leq\ell_F$, we have 
$\|x^*\| \leq \frac{C\|z^*\|_\infty}{(1-\gamma)^2}\leq \frac{C\ell_F}{(1-\gamma)^2}$ and thus $\|\hat x\| \leq  \frac{C\ell_F}{(1-\gamma)^2}$ with probability 1. Therefore, adding a constraint that $\|x\|\leq \frac{C\ell_F}{(1-\gamma)^2}$ will not change the solutions of problems \eqref{prob:SSP} and \eqref{prob:ESP}. Formally speaking, we will then apply the theory of \cite{Gen_SPP} to the following pair of constrained problems:
\begin{equation}
\label{prob:SSP'}
(x^*,z^*) = \argmax_{x\in\cX}\argmin_{z\in\mathcal{Z}} V(\theta;z) + \delta\nabla_\theta V(\theta;z)^\top x - F^*(z) - \frac{\delta}{2}\|x\|^2,
\end{equation}
and
\begin{equation}
\label{prob:ESP'}
(\hat x ,\hat z ) = \argmax_{ x\in\cX}\argmin_{z\in\mathcal{Z}} \tilde V(\theta;z) + \delta\nabla_\theta \tilde V(\theta;z)^\top x - F^*(z)- \frac{\delta}{2}\|x\|^2.
\end{equation}
with $\cX = \{x: \|x\|\leq\frac{C\ell_F}{(1-\gamma)^2}\}$ and $\mathcal{Z} = \{z:\|z\|_\infty\leq\ell_F\}$. The problems \eqref{prob:SSP} and \eqref{prob:SSP'} share the same solution, and problems \eqref{prob:ESP} and \eqref{prob:ESP'} share the same solution.  

Finally, similar to the proof of \eqref{thm-pgestimate-3}, for any $x\in\cX$ and $z\in\mathcal{Z}$
$$V(\theta;z) +\delta\nabla_{\theta}V(\theta;z)^\top x - \EE\left[\tilde V(\theta;z;\zeta_i) + \delta\nabla_{\theta}\tilde{V}(\theta;z;\zeta_i)^\top x\right] =  \cO\left(\frac{\gamma^K}{1-\gamma}\right).$$
For the simplicity of discussion, let us assume that $K$ is large enough so that we can ignore the $\cO\left(\frac{\gamma^K}{1-\gamma}\right)$ bias. Therefore problem \eqref{prob:ESP'} can be viewed as an empirical version of the problem \eqref{prob:SSP'} with negligible bias. To apply the theory of \cite{Gen_SPP}, define 
$$\Psi_\zeta(x,z) := \tilde V(\theta;z;\zeta) + \delta\nabla_\theta \tilde V(\theta;z;\zeta)^\top x - F^*(z)- \frac{\delta}{2}\|x\|^2.$$
Then for any sample path $\zeta$, $\Psi_\zeta$ satisfies the following set of properties:
\begin{itemize}
	\item $\Psi_\zeta(\cdot,z)$ is $\mu_x$-strongly concave under $L_2$ norm. And $\Psi_\zeta(x,\cdot)$ is $\mu_z$-strongly convex under the $L_\infty$ norm. In other words, for $\forall x,x'\in\cX$ and $z,z'\in\mathcal{Z}$,
	\begin{equation*}
	\begin{cases}
	\Psi_\zeta(x',z)\geq\Psi_\zeta(x,z)  + \langle u,x'-x\rangle + \frac{\mu_x}{2}\|x'-x\|^2, &  u\in\partial_x\Psi_\zeta(x,z),\\
	\Psi_\zeta(x,z')\leq\Psi_\zeta(x,z) + \langle v,z'-z\rangle - \frac{\mu_z}{2}\|z'-z\|_\infty^2,& v\in\partial_z \Psi_\zeta(x,z).
	\end{cases}
	\end{equation*} 
	In our case, it is clear that $\mu_x = \delta$. Due to Theorem 3 of \cite{kakade2012regularization}, $\mu_z = L_F^{-1}$.
	\item The feasible regions $\cX$ and $\mathcal{Z}$ are compact convex sets. For every $\zeta$,  there exist constants $\ell_x(\xi,z)$ and $\ell_z(\xi,x)$ s.t. 
	\begin{equation*}
	\begin{cases}
	|\Psi_\zeta(x',z)-\Psi_\zeta(x,z)| \leq \ell_x(\zeta,z)\|x'-x\|,& \forall x,x'\in\cX \,\,\mbox{and}\,\,y\in\cY,\\
	|\Psi_\zeta(x,z')-\Psi_\zeta(x,z)|\leq \ell_z(\zeta,x)\|z'-z\|_\infty,& \forall z,z'\in\mathcal{Z}\,\,\mbox{and}\,\,x\in\cX.
	\end{cases}
	\end{equation*} 
	In our case, we gave $\ell_z(\zeta,x) = \sup\left\{\|u\|_1: z\in\mathcal{Z}, u\in\partial_z \Psi_\zeta(x,z)\right\} = \frac{1}{1-\gamma} + \ell_{F^*} + \cO(\delta)$ and 
	$\ell_x(\zeta,z) = \sup_{x\in\cX}\|\nabla_x \Psi_\zeta(x,z)\|  = \cO(\delta)$. Consequently,  
	\begin{equation*} 
	\begin{cases}
	(\ell_x^w)^2:=\sup_{z\in\mathcal{Z}}\EE\big[\ell^2_x(\zeta,z)\big] = \cO(\delta^2),\\
	(\ell_z^w)^2:=\sup_{x\in\cX}\EE\big[\ell^2_z(\zeta,x)\big] = \cO(\ell_{F^*}^2 + \frac{1}{(1-\gamma)^2} + \delta^2).
	\end{cases}
	\end{equation*}
\end{itemize}
With the above two properties, Theorem 1 of \cite{Gen_SPP} indicates that 
$$\frac{\mu_z}{2}\EE\left[\|\hat z-z^*\|^2_\infty\right]\leq\frac{2\sqrt{2}}{n}\cdot\left( \frac{(\ell^w_x)^2}{\mu_x} + \frac{(\ell_z^w)^2}{\mu_z}\right).$$
With the detailed parameters substituted in the above inequality, we have 
$$\EE\left[\|\hat z-z^*\|^2_\infty\right]\leq\cO\left(\frac{L_F^2}{n(1-\gamma)^2} + \frac{L_F^2\ell_{F^*}^2}{n} + \frac{L_F^2\delta^2 + L_F\delta}{n}\right)$$

as stated in Lemma  \ref{lemma:theorem2}\ref{lemma:step3}. \hfill $\qed$

\section{Proof of Theorem \ref{theorem:global-opt}}\label{appdx:theorem:global-opt}
\begin{proof}
Let $\theta^*$ be a first-order stationary solution of \eqref{prob:dual-para-pi}.
	When $F$ is concave and locally Lipschitz continuous in a neighbourhood containing $\lambda(\Theta)$,
we can  compute the Fr\'echet superdifferential of $F\circ\lambda$ at $\theta^*$ by the chain rule, see \cite{drusvyatskiy2019efficiency}. That is 
	$$\hat{\partial}(F\circ\lambda)(\theta^*) = \left[\nabla_\theta\lambda(\theta^*)\right]^\top\partial F(\lambda^*)$$ where $\partial \F(\lambda^*)$ denotes the set of supergradients of the concave function $\F$ at $\lambda^*$. Then there exists $w^*\in\partial F(\lambda^*)\in\RR^{SA}$ such that $u^* := [\nabla_\theta\lambda(\theta^*)]^\top w^* \in \hat{\partial}(F\circ\lambda)(\theta^*)$ as in \eqref{defn:1st-order-condition}. 
It follows from \eqref{defn:1st-order-condition} that
	\begin{equation}	\label{thm:global-opt-1}
	\langle w^*,\nabla_{\theta }\lambda(\theta^*)(\theta-\theta^*)\rangle \leq 0, \quad\mbox{for}\quad \forall \theta\in\Theta.
	\end{equation}
	For any $ \lambda\in\lambda(\Theta)$, we let $\theta := g(\lambda)$ such that $\lambda = \lambda(\theta)$. Therefore, by adding and subtracting $\nabla_{\theta}\lambda(\theta^*) \theta $ inside the inner product we have 
	\begin{eqnarray}
	\label{eqn:scboy-1}
	\langle w^*,\lambda-\lambda^*\rangle
	& = & \langle w^*,\lambda(\theta)-\lambda(\theta^*)\rangle \\
	& = & \langle w^*,\nabla_{\theta}\lambda(\theta^*)(\theta-\theta^*)\rangle  + \langle w^*,\lambda(\theta)-\lambda(\theta^*)-\nabla_{\theta}\lambda(\theta^*)(\theta-\theta^*)\rangle\nonumber\\
	& \leq & 0 + \|w^*\|\|\lambda(\theta)-\lambda(\theta^*)-\nabla_{\theta}\lambda(\theta^*)(\theta-\theta^*)\|\nonumber.	
	\end{eqnarray}
	where in the last inequality we group terms and apply Cauchy-Schwartz.
	Note that the Jacobian matrix $\nabla_{\theta}\lambda(\theta)$ is Lipschitz continuous. Denote the Lipschitz constant by $L_\lambda$, i.e., $\|\nabla_{\theta}\lambda(\theta) - \nabla_{\theta}\lambda(\theta')\|\leq L_\lambda\|\theta-\theta'\|$ for all $\theta,\theta'\in\Theta$. Then, $$\|\lambda(\theta)-\lambda(\theta^*)-\nabla_{\theta}\lambda(\theta^*)(\theta-\theta^*)\|\leq\frac{L_\lambda}{2}\|\theta-\theta^*\|^2.$$
	By Assumption \ref{assumption:gen-para}, we know 
	$$\|\theta-\theta^*\|^2 = \|g(\lambda) - g(\lambda^*)\|^2 \leq \ell_{\theta}^2\normm{\lambda-\lambda^*}^2.$$
	Substituting the above inequalities into \eqref{eqn:scboy-1} yields
	\begin{equation}
	\label{thm:global-opt-2}
	\langle w^*,\lambda-\lambda^*\rangle \leq \frac{L_\lambda\ell_{\theta}^2}{2}\|w^*\|\normm{\lambda-\lambda^*}^2\quad\quad \forall \lambda\in\lambda(\Theta).
	\end{equation}
	Note that \eqref{thm:global-opt-2} holds for arbitrary $\lambda\in\lambda(\Theta)$. Therefore, since $\lambda(\Theta)$ is assumed to be convex (Assumption \ref{assumption:gen-para}(i)), we can also substitute $\lambda$ with $(1-\alpha)\lambda^* + \alpha\lambda, \alpha\in[0,1]$ into the above equation, which yields
	$$\alpha\langle w^*,\lambda-\lambda^*\rangle \leq \frac{L_\lambda\ell_{\theta}^2\alpha^2}{2}\|w^*\|\normm{\lambda-\lambda^*}^2\quad\quad \forall \lambda\in\cL, \forall \alpha\in[0,1].$$
	Divide both sides of the preceding expression by $\alpha$ and take $\alpha\rightarrow0+$ gives
	$$\langle w^*,\lambda-\lambda^*\rangle \leq \lim\limits_{\alpha\rightarrow0+} \frac{L_\lambda\ell_{\theta}^2\alpha}{2}\|w^*\|\normm{\lambda-\lambda^*}^2 = 0 \quad \quad \forall \lambda\in\lambda(\Theta).$$
	Recall that the following problem is concave in $\lambda$:
	$$\max_{\lambda}\quad\!\! F(\lambda)\qquad\mbox{s.t.}\qquad \lambda\in\lambda(\Theta),$$
	therefore we conclude that $\lambda^*$ is the global optimal solution. Then $\theta^* = g(\lambda^*)$ is the globally optimal solution of the nonconvex optimization problem \eqref{prob:dual-para-pi}.
\end{proof}

\section{Proof of Theorem \ref{theorem:iteration complexity-gen}}
\label{appdx:theorem:iteration complexity}
\subsection{Proof of sublinear convergence}
\begin{proof}
First, the Lipschitz continuity in Assumption \ref{assumption:ncvx-Lip} indicates that 
	$$\left|\F(\lambda(\theta)) - \F(\lambda(\theta^k)) - \langle \nabla_\theta\F(\lambda(\theta^k)),\theta-\theta^k\rangle\right|\leq \frac{L}{2}\|\theta-\theta^k\|^2.$$
	Consequently, for any $\theta\in\Theta$ we have the ascent property:
	\begin{equation}\label{eq:taylor_proof}
	\F(\lambda(\theta)) \geq \F(\lambda(\theta^k)) + \langle \nabla_\theta\F(\lambda(\theta^k)),\theta-\theta^k\rangle - \frac{L}{2}\|\theta-\theta^k\|^2 \geq \F(\lambda(\theta)) - L\|\theta-\theta^k\|^2.
	\end{equation}
	The optimality condition in the policy update rule \eqref{defn:grad-proj} then yields
\begin{align}
	\label{thm:ItrCmp-1}
	\MoveEqLeft
	\F(\lambda(\theta^{k+1}))  \geq  \F(\lambda(\theta^k)) + \langle \nabla_\theta\F(\lambda(\theta^k)),\theta^{k+1}-\theta^k\rangle - \frac{L}{2}\|\theta^{k+1}-\theta^k\|^2 \nonumber \\
	& =  \max_{\theta\in\Theta} \F(\lambda(\theta^k)) + \langle \nabla_\theta\F(\lambda(\theta^k)),\theta-\theta^k\rangle - \frac{L}{2}\|\theta-\theta^k\|^2\nonumber\\
	& \overset{\text{(a)}}{\geq}  \max_{\theta\in\Theta} \F(\lambda(\theta)) - L\|\theta-\theta^k\|^2\nonumber\\
	& \overset{\text{(b)}}{\geq}   \max_{\alpha\in[0,1]}\left\{\F(\lambda(\theta_{\alpha})) - L\|\theta_{\alpha}-\theta^k\|^2: \theta_{\alpha} = g(\alpha\lambda(\theta^*) + (1-\alpha)\lambda(\theta^k)) \right\}.
\end{align}
	Here, step (a) is due to \eqref{eq:taylor_proof} and step (b) uses the convexity of $\lambda(\Theta)$. Now, we proceed to analyze the right-hand side of \eqref{thm:ItrCmp-1}. First, by the concavity of $\F$ and the fact that $\lambda\circ g = id$, we have 
	$$\F(\lambda(\theta_{\alpha})) = \F(\alpha\lambda(\theta^*) + (1-\alpha)\lambda(\theta^k))\geq\alpha\F(\lambda(\theta^*)) + (1-\alpha)\F(\lambda(\theta^k)).$$
	Moreover, by the Lipschitz continuity assumption of $g$, we have 
	\begin{eqnarray}
	\label{eqn:important-gen}
		\|\theta_{\alpha} - \theta^k\|^2 & = & \|g(\alpha\lambda(\theta^*) + (1-\alpha)\lambda(\theta^k))- g(\lambda(\theta^k))\|^2\\
		& \leq & \alpha^2\ell_{\theta}^2\normm{\lambda(\theta^*) - \lambda(\theta^k)}^2\nonumber\\
		& \leq & \alpha^2\ell_{\theta}^2D_\lambda^2.\nonumber
	\end{eqnarray}
	Substituting the above two inequalities into the right-hand side of \eqref{thm:ItrCmp-1}, we get 
	\begin{align}
	\MoveEqLeft 
	\F(\lambda(\theta^*)) - \F(\lambda(\theta^{k+1})) \nonumber \\
	& \leq  \min_{\alpha\in[0,1]}\left\{\F(\lambda(\theta^*))-\F(\lambda(\theta_{\alpha})) + L\|\theta_{\alpha}-\theta^k\|^2: \theta_{\alpha} = g(\alpha\lambda(\theta^*) + (1-\alpha)\lambda(\theta^k)) \right\}\nonumber\\
	& \leq  \min_{\alpha\in[0,1]}(1-\alpha)\big(\F(\lambda(\theta^*))-\F(\lambda(\theta^k))\big) + \alpha^2L\ell_{\theta}^2D_\lambda^2 \,.
	\label{thm:ItrCmp-2-gen}
	\end{align}
	Let $\alpha_k = \frac{\F(\Lambda(\pi^*)) - \F(\Lambda(\pi^k))}{2L\ell_{\theta}^2D_\lambda^2}\geq0$, which is the minimizer of the RHS of  \eqref{thm:ItrCmp-2-gen} as long as it satisfies $\alpha_k\leq 1$. 
	
	Now, we claim the following: If $\alpha_k\ge 1$ then $\alpha_{k+1}<1$. Further, if $\alpha_k<1$ then $\alpha_{k+1}\le \alpha_k$. The two claims together mean that $(\alpha_k)_k$ is decreasing and all $\alpha_k$ are in $[0,1)$ except perhaps $\alpha_0$.

	To prove the first of the two claims, assume $\alpha_k\ge 1$.
	This implies that $\F(\Lambda(\pi^*)) - \F(\Lambda(\pi^k))\geq 2L\ell_{\theta}^2D_\lambda^2$. Hence, choosing $\alpha=1$ in \eqref{thm:ItrCmp-2-gen}, we get
	\[\F(\lambda(\theta^*)) - \F(\lambda(\theta^{k}))\leq L\ell_{\theta}^2D_\lambda^2\]
	which implies that $\alpha_{k+1}\le 1/2<1$.
	
	To prove the second claim, we plug  $\alpha_k$ into  \eqref{thm:ItrCmp-2-gen} to get
	\[
	\F(\lambda(\theta^*)) - \F(\lambda(\theta^{k+1})) \leq  \left(1-\frac{\F(\lambda(\theta^*)) - \F(\lambda(\theta^{k}))}{4L\ell_{\theta}^2D_\lambda^2}\right)(\F(\lambda(\theta^*)) - \F(\lambda(\theta^{k}))),
	\]
	which shows that $\alpha_{k+1}\le \alpha_k$ as required.
	
	Now, by our preceding discussion, for $k=1,2,\dots$ the previous recursion holds.
	Using the definition of $\alpha_k$, we rewrite this in the equivalent form  
	\[
	\frac{\alpha_{k+1}}{2}\leq \left(1-\frac{\alpha_{k}}{2}\right)\cdot\frac{\alpha_{k}}{2}.
	\] 
By rearranging the preceding expressions and algebraic manipulations, we obtain
	$$\frac{2}{\alpha_{k+1}} \geq \frac{1}{\left(1-\frac{\alpha_{k}}{2}\right)\cdot\frac{\alpha_{k}}{2}} = \frac{2}{\alpha_{k}} + \frac{1}{1-\frac{\alpha_{k}}{2}}\geq\frac{2}{\alpha_k} + 1.$$
	For simplicity assume that $\alpha_0<1$ also holds. Then,
    $\frac{2}{\alpha_{k}}\geq \frac{2}{\alpha_0} + k$, and consequenlty
	$$\F(\lambda(\theta^*)) - \F(\lambda(\theta^{k}))\leq \frac{\F(\lambda(\theta^*)) - \F(\lambda(\theta^0))}{1+ \frac{\F(\lambda(\theta^*)) - \F(\lambda(\theta^0))}{4L\ell_{\theta}^2D_\lambda^2}\cdot k} \leq \frac{4L\ell_{\theta}^2D_\lambda^2}{k}.$$
	A similar analysis holds when $\alpha_0>1$. Combining these two gives that 
	$\F(\lambda(\pi^*)) - \F(\lambda(\pi^{k}))\leq \frac{4L\ell_{\theta}^2D_\lambda^2}{k+1}$ no matter the value of $\alpha_0$, which proves the result. 
\end{proof}


\subsection{Proof of exponential convergence}
When the strong concavity of $F$ is available, we further provide the exponential convergence result. 
\begin{proof}
	We start from \eqref{thm:ItrCmp-1} whose proof requires no assumption on strong concavity of $F$, which is
	\begin{equation}
	\label{thm:ItrCmp-1'}
	\F(\lambda(\theta^{k+1})) \geq \max_{\alpha\in[0,1]}\left\{\F(\lambda(\theta_{\alpha})) - L\|\theta_{\alpha}-\theta^k\|^2: \theta_{\alpha} = g(\alpha\lambda(\theta^*) + (1-\alpha)\lambda(\theta^k)) \right\}.
	\end{equation}
	By the $\mu$-strong concavity of $F$, we have 
	$$\F(\lambda(\theta_{\alpha})) = \F(\alpha\lambda(\theta^*) + (1-\alpha)\lambda(\theta^k))\geq\alpha\F(\lambda(\theta^*)) + (1-\alpha)\F(\lambda(\theta^k)) + \frac{\mu}{2}\alpha(1-\alpha)\normm{\lambda(\theta^*)-\lambda(\theta^k)}^2.$$
	By the Lipschitz continuity of $g$, we know that 
	\begin{eqnarray*}
		\|\theta_{\alpha} - \theta^k\|  =  \|g(\alpha\lambda(\theta^*) + (1-\alpha)\lambda(\theta^k))- g(\lambda(\theta^k))\|
		\leq  \alpha\ell_\theta\normm{\lambda(\theta^*) - \lambda(\theta^k)}
	\end{eqnarray*}
	Substituting the above two inequalities into the right-hand side of \eqref{thm:ItrCmp-1'}, we get 
	\begin{eqnarray}
	\label{thm:ItrCmp-2'}
	& &\F(\lambda(\theta^*)) - \F(\lambda(\theta^{k+1})) \\
	& \leq & \min_{\alpha\in[0,1]}\left\{\F(\lambda(\theta^*))-\F(\lambda(\theta_{\alpha})) + L\|\theta_{\alpha}-\theta^k\|^2: \theta_{\alpha} = g(\alpha\lambda(\theta^*) + (1-\alpha)\lambda(\theta^k)) \right\}\nonumber\\
	& \leq & \min_{\alpha\in[0,1]}(1-\alpha)\big(\F(\lambda(\theta^*))-\F(\lambda(\theta^k))\big) - \alpha\left(\frac{1-\alpha}{2}\mu-L\ell_\theta^2\alpha\right)\normm{\lambda(\theta^*) - \lambda(\theta^k)}^2 \nonumber
	\end{eqnarray}
	Suppose we choose $\bar\alpha = \frac{1}{1+L\ell^2_\theta/\mu}<1$ such that $\left(\frac{1-\bar\alpha}{2}\mu-L\ell_\theta^2\bar\alpha\right)=0$. Then we have a contraction with modulus $1-\bar{\alpha}$ as
	$$\F(\lambda(\theta^*)) - \F(\lambda(\theta^{k+1})) \leq (1-\bar{\alpha})\F(\lambda(\theta^*)) - \F(\lambda(\theta^{k})).$$
	Consequently,  for any $k\geq1$, we have
	$$\F(\lambda(\theta^*)) - \F(\lambda(\theta^{k})) \leq (1-\bar{\alpha})^k\left(\F(\lambda(\theta^*)) - \F(\lambda(\theta^{0}))\right).$$
	which can be translated into iteration complexity by fixing $\epsilon$ and initialization $\theta^0$, and solving for the minimal $k$ such that  $\F(\lambda(\theta^*)) - \F(\lambda(\theta^{k}))\leq \epsilon$. Doing so is an algebraic exercise which results in
	$$\cO\left(\frac{1}{\bar{\alpha}}\log\left(\frac{\F(\lambda(\theta^*)) - \F(\lambda(\theta^{0}))}{\epsilon}\right)\right) = \cO\left(\frac{L{\ell}_\theta^2}{\mu}\log\left(\frac{1}{\epsilon}\right)\right)$$
\end{proof}

%
%
%
%
%
%
\section{Validating Assumption \ref{assumption:gen-para} for tabular policy case}
\label{appdx:bijection}
For the tabular policy case, the following Proposition holds true and hence the Assumption \ref{assumption:gen-para} is satisfied in this case. 
\begin{proposition}
	\label{proposition:bijection}
	Suppose $\xi_s>0$ for $\forall s\in\cS$. Then the following hold: 
	\begin{itemize}
		\item[(i).] The mappings $\Pi$ and $\Lambda$ form a pair of bijections between the convex sets $\das$ and $\cL$;	\\
		\item[(ii).] $\exists L_\lambda>0$ s.t. $\|\nabla\Lambda(\pi)-\nabla\Lambda(\pi')\| \leq L_\lambda\|\pi-\pi'\|, \forall \pi,\pi'\in\das$;
		\item[(iii).] For all $\lambda,\lambda'\in\cL$,
		we have 
		$$
		\|\Pi(\lambda) - \Pi(\lambda')\|^2\leq 2\sum_s 
		\Big(\sum_a (\lambda'_{sa} - \lambda_{sa})^2+
		(\sum_a \lambda'_{sa}-\lambda_{sa})^2\Big)/\big(\sum_a \lambda_{sa}\big)^{2}\!.
		$$
		Consequently, 
		$\|\Pi(\lambda) - \Pi(\lambda')\|\leq \frac{2}{\min_s \xi_s}\|\lambda-\lambda'\|_1$
	\end{itemize}  
\end{proposition} 

\begin{proof}$~$\vspace{0.1cm}\\
\textbf{Proof of (i)}:
	The equations $\Pi\circ\Lambda = \mathrm{id}_{\cL}$ and $\Lambda\circ\Pi = \mathrm{id}_{\das}$ are standard. See, e.g., \cite{altman1999constrained} or Appendix A of \cite{zhang2020cautious}.  \vspace{0.3cm}\\
\textbf{Proof of (ii)}:
	For the existence of the $L_\lambda$-Lipschitz constant of the gradient $\nabla\Lambda$, note that 
	the $t$-th term of the infinite sum 
	\[
	\Lambda_{sa}(\pi) = \sum_{t=0}^\infty \gamma^t\cdot\mathbb{P}\bigg(s_t = s, a_t = a\,\,\bigg|\,\,\pi, s_0\sim\xi \bigg)
	\]
	is a $(t+1)$-th order polynomial. Therefore, $\Lambda_{sa}(\pi)$ can actually be defined for any $\pi$ even if $\pi\notin\das$, as long as this infinite series of polynomial of $\pi$ converges absolutely. Note that for $\forall \pi\in\das$, since $0\leq\mathbb{P}\big(s_t = s, a_t = a\,\,\big|\,\,\pi, s_0\sim\xi \big)\leq1$ this infinite series is absolutely convergent. Because we have $0<\gamma<1$, even if we slightly purterb the $\pi$ within a neighbourhood of it (not necessarily in $\das$ after purterbation), the infinite series is still absolutely convergent. This indicates that  $\Lambda_{sa}$ is infinitely continuously differentiable in an open neighbourhood containing $\das$, then due to the compactness of $\das$, we are able to argue that there exists a $L_\lambda$ s.t. $\nabla\Lambda$ is $L_\lambda$-Lipschitz continuous within $\das$.

\textbf{Proof of (iii)}:
	Now, we provide the calculation of the Lipschitz constant of $\Pi$. For the ease of notation, let us define $\mu_s = \sum_{a\in\cA} \lambda_{sa}$ and $\mu_s' = \sum_{a\in\cA}\lambda_{sa}'$. Then for  $\forall \lambda,\lambda'\in\cL$ and $\forall (s,a)\in\cS\times\cA$, it holds that
	\begin{eqnarray*}
		\Pi_{sa}(\lambda) - \Pi_{sa}(\lambda') & = & \frac{\lambda_{sa}}{\mu_s} -\frac{\lambda_{sa}'}{\mu_s'}\\
		& = & \left(\frac{\lambda_{sa}}{\mu_s} -\frac{\lambda_{sa}'}{\mu_s}\right) + \left(\frac{\lambda_{sa}'}{\mu_s} -\frac{\lambda_{sa}'}{\mu_s'}\right)\\
		& = & \frac{1}{\mu_s}(\lambda_{sa}-\lambda_{sa}') + \frac{\mu_s'-\mu_s}{\mu_s\mu_s'}\lambda_{sa}'. 
	\end{eqnarray*}
	Consequently, we can compute the norm difference of the preceding expression and apply the triangle inequality:
	\begin{eqnarray}
	\label{prop:bijection-1}
	\|\Pi(\lambda) - \Pi(\lambda')\|^2 & = & \sum_{s\in\cS}\sum_{a\in\cA} \left(\Pi_{sa}(\lambda) - \Pi_{sa}(\lambda')\right)^2\\
	& \leq & 2\sum_{s\in\cS}\sum_{a\in\cA} \frac{1}{\mu_s^2}(\lambda_{sa}-\lambda_{sa}')^2 +2\sum_{s\in\cS}\sum_{a\in\cA} \frac{(\mu_s'-\mu_s)^2}{\mu_s^2(\mu_s')^2}(\lambda_{sa}')^2\nonumber\\
	&\leq & 2\sum_{s\in\cS} \frac{1}{\mu_s^2}\left(\sum_{a\in\cA}(\lambda_{sa}-\lambda_{sa}')^2 + (\mu_s'-\mu_s)^2\right)\,,\nonumber
	\end{eqnarray}
	where the last inequality follows because $\norm{x}_2^2 \le \norm{x}_1^2$ holds for any vector $x$ (here, $\norm{\cdot}_p$ denotes the $p$-norm). Finally, note that 
	$\mu_s\geq\xi_s>0$, we have 
	\begin{eqnarray} 
	\|\Pi(\lambda) - \Pi(\lambda')\|^2 &\leq & 2\sum_{s\in\cS} \frac{1}{\mu_s^2}\left(\sum_{a\in\cA}(\lambda_{sa}-\lambda_{sa}')^2 + (\mu_s'-\mu_s)^2\right)\nonumber\\
	& \leq & \frac{2}{\min_s\xi_s^2}\sum_{s\in\cS} \left(\sum_{a\in\cA}(\lambda_{sa}-\lambda_{sa}')^2 + \big(\sum_{a\in\cA}|\lambda_{sa}-\lambda_{sa'}|\big)^2\right)\nonumber\\
	& \leq & \frac{4}{\min_s\xi_s^2}\|\lambda-\lambda'\|_1^2\nonumber
	\end{eqnarray}
	Take the square root of both sides completes the proof. 
\end{proof}

\section{Proof of Theorem \ref{theorem:iteration complexity}}
\begin{proof}
To prove this theorem, it suffices to observe that \eqref{thm:ItrCmp-1} is still true with $\theta = \pi$, $\lambda(\theta) = \Lambda(\pi)$ and $g(\lambda) = \Pi(\lambda)$. Therefore, \eqref{thm:ItrCmp-1} can be translated as 
\begin{equation}
\label{thm:ItrCmp-1''}
\F(\Lambda(\pi^{k+1})) \geq \max_{\alpha\in[0,1]}\left\{\F(\Lambda(\pi_{\alpha})) - L\|\pi_{\alpha}-\pi^k\|^2: \pi_{\alpha} = \Pi(\alpha\Lambda(\pi^*) + (1-\alpha)\Lambda(\pi^k)) \right\}.
\end{equation}
By the concavity of $\F$ and the fact that $\Lambda\circ \Pi = id$, we have  
\begin{equation}
\label{eqn:-1}
\F(\Lambda(\pi_{\alpha})) = \F(\alpha\Lambda(\pi^*) + (1-\alpha)\Lambda(\pi^k))\geq\alpha\F(\Lambda(\pi^*)) + (1-\alpha)\F(\Lambda(\pi^k)).
\end{equation}
For the inequality \eqref{eqn:important-gen}, we can derive a tighter bound by the following argument:
\begin{eqnarray}
\label{eqn:important}
\|\pi_{\alpha} - \pi^k\|^2 & = & \|\Pi(\alpha\Lambda(\pi^*) + (1-\alpha)\Lambda(\pi^k))- \Pi(\Lambda(\pi^k))\|^2\\
& \leq & \alpha^2\sum_s\frac{1}{\big(\sum_a \lambda_{sa}\big)^2}\left(\sum_a (\lambda^*_{sa} - \lambda_{sa})^2 + \big(\sum_a \lambda^*_{sa} - \sum_a \lambda_{sa}\big)^2\right)\nonumber\\
& \leq & 4\alpha^2\sum_s\frac{1}{\big(\sum_a \lambda_{sa}\big)^2}\left(\big(\sum_a \lambda^*_{sa}\big)^2+\big(\sum_a \lambda_{sa}\big)^2\right)\nonumber\\
& = & 4\alpha^2\sum_s\frac{\big(d_\xi^{\pi^*}(s)\big)^2+\big(d_\xi^{\pi^k}(s)\big)^2}{\big(d_\xi^{\pi^k}(s)\big)^2}\nonumber\\
& = & 4\alpha^2|\cS| + 4\alpha^2\sum_s\left(\frac{d_\xi^{\pi^*}(s)}{d_\xi^{\pi^k}(s)}\right)^2 \nonumber\\
& \leq & 4\alpha^2|\cS| + 4\alpha^2|\cS|\left\|\frac{d_\xi^{\pi^*}}{d_\xi^{\pi^k}}\right\|_\infty^2\nonumber\\
& \leq & 4\alpha^2|\cS|\cdot\left(1+(1-\gamma)^{-2}\left\|d_\xi^{\pi^*}/\xi\right\|_\infty^2\right)\nonumber\\
&\leq&\frac{5\alpha^2|\cS|}{(1-\gamma)^{2}}\left\|d_\xi^{\pi^*}/\xi\right\|_\infty^2\nonumber
\end{eqnarray}
Denote $D: = \frac{5|\cS|}{(1-\gamma)^{2}}\left\|d_\xi^{\pi^*}/\xi\right\|_\infty^2$. Substituting the above two inequalities into the right-hand side of \eqref{thm:ItrCmp-1''}, we get 
\begin{align}
\MoveEqLeft 
\F(\Lambda(\pi^*)) - \F(\Lambda(\pi^{k+1})) \nonumber \\
& \leq  \min_{\alpha\in[0,1]}\left\{\F(\Lambda(\pi^*))-\F(\Lambda(\pi_{\alpha})) + L\|\pi_{\alpha}-\pi^k\|^2: \pi_{\alpha} = \Pi(\alpha\Lambda(\pi^*) + (1-\alpha)\Lambda(\pi^k)) \right\}\nonumber\\
& \leq  \min_{\alpha\in[0,1]}(1-\alpha)\big(\F(\Lambda(\pi^*))-\F(\Lambda(\pi^k))\big) + LD\alpha^2 \,.
\label{thm:ItrCmp-2}
\end{align}
Note that \eqref{thm:ItrCmp-2} differs from \eqref{thm:ItrCmp-2-gen} by replacing $\ell_{\theta}^2D_\lambda^2$ with $D$. The latter proof of Theorem \ref{theorem:iteration complexity} is almost identical to that of Theorem \ref{theorem:iteration complexity-gen} and hence we omit the proof.
\end{proof}


\end{document}